\documentclass[journal]{IEEEtran}

\usepackage{cite}
\usepackage{amsmath,amssymb,amsfonts,bm}
\usepackage{graphicx}
\usepackage{textcomp}
\usepackage{booktabs}
\usepackage{multirow}
\usepackage{makecell}
\usepackage{array}
\usepackage{tabularx}
\usepackage{placeins}
\usepackage{microtype}
\usepackage{url}
\usepackage{hyperref}

\hypersetup{hidelinks}
\graphicspath{{image/}}
\newcolumntype{Y}{>{\raggedright\arraybackslash}X}

\newcommand{\safeincludegraphics}[2][]{%
  \IfFileExists{#2}{\includegraphics[#1]{#2}}{%
    \IfFileExists{image/#2}{\includegraphics[#1]{#2}}{%
      \fbox{%
        \begin{minipage}[c][0.22\textheight][c]{0.95\linewidth}
        \centering
        \texttt{\detokenize{#2}}\\[0.5em]
        \small Placeholder figure. Replace this box with the final vector or raster figure.
        \end{minipage}%
      }%
    }%
  }%
}

\newcommand{\I}{\mathbf{I}}

\newcommand{\Yzero}{\mathbf{Y}_{0}}
\newcommand{\Yt}{\mathbf{Y}_{t}}
\newcommand{\Yhat}{\hat{\mathbf{Y}}_{0}}
\newcommand{\Zzero}{\mathbf{Z}_{0}}

\newcommand{\eps}{\boldsymbol{\epsilon}}
\newcommand{\etaNoise}{\boldsymbol{\eta}}

\newcommand{\auditname}[1]{\ifmmode\text{#1}\else #1\fi}

\newcommand{\randomauditfull}{Noisy-State Randomization Audit}
\newcommand{\pairingauditfull}{Cross-Case Pairing Audit}
\newcommand{\corecounterfactualfull}{Image-Only Core Counterfactual}
\newcommand{\randyt}{Random-\ensuremath{Y_t}}
\newcommand{\shuffleyt}{Shuffle-\ensuremath{Y_t}}
\newcommand{\corenodiff}{Core-No-Diff}

\newcommand{\pathSegTargeted}{Seg-Targeted}

\newcommand{\figref}[1]{Fig.~\ref{#1}}
\newcommand{\tabref}[1]{Table~\ref{#1}}
\newcommand{\paperref}[3]{#1~\cite{#3}}
\newcommand{\methodMedSegDiffVTwo}{\paperref{MedSegDiff-V2}{https://doi.org/10.1609/aaai.v38i6.28418}{wu2023medsegdiffv2}}
\newcommand{\methodUniSegDiff}{\paperref{UniSegDiff}{https://doi.org/10.1007/978-3-032-04937-7_63}{hu2025unisegdiff}}

\newcommand{\methodSSB}{\paperref{SSB}{https://doi.org/10.1007/978-3-032-04965-0_3}{baru2025ssb}}
\newcommand{\methodEIDiffSeg}{\paperref{EIDiffSeg}{https://doi.org/10.1016/j.knosys.2024.112426}{xia2024eidiffseg}}
\newcommand{\methodMoDiff}{\paperref{MoDiff}{https://doi.org/10.1007/978-3-032-04947-6_37}{ahn2025modiff}}
\newcommand{\methodCIMD}{\paperref{CIMD}{https://doi.org/10.1109/CVPR52729.2023.01110}{rahman2023ambiguous}}
\newcommand{\methodLDSeg}{\paperref{LDSeg}{https://doi.org/10.1016/j.bspc.2025.109380}{zaman2026ldseg}}
\newcommand{\methodSDSeg}{\paperref{SDSeg}{https://doi.org/10.1007/978-3-031-72111-3_62}{lin2024sdseg}}
\newcommand{\methodTSLDSeg}{\paperref{TSLDSeg}{https://doi.org/10.1016/j.patcog.2025.112795}{yang2026tsldseg}}
\newcommand{\methodEnsemDiff}{\paperref{EnsemDiff}{https://proceedings.mlr.press/v172/wolleb22a.html}{wolleb2022ensemdiff}}
\newcommand{\methodMedSegDiffVOne}{\paperref{MedSegDiff-V1}{https://proceedings.mlr.press/v227/wu24a.html}{wu2023medsegdiff}}
\newcommand{\methodBerDiff}{\paperref{BerDiff}{https://doi.org/10.1007/978-3-031-43901-8_47}{chen2023berdiff}}
\newcommand{\methodRetiDiff}{\paperref{RetiDiff}{https://doi.org/10.1007/978-3-032-04937-7_49}{li2025retidiff}}
\newcommand{\methodACPDiff}{\paperref{ACP-Diff}{https://doi.org/10.1007/978-3-031-43901-8_52}{amit2023acpdiff}}
\newcommand{\methodDiffMedSeg}{\paperref{Diff-MedSeg}{https://doi.org/10.1109/TMM.2026.3702543}{liu2026diffmedseg}}
\newcommand{\methodSegDiff}{\paperref{SegDiff}{https://doi.org/10.48550/arXiv.2112.00390}{amit2021segdiff}}

\newcommand{\methodCDAL}{\paperref{cDAL}{https://doi.org/10.1007/978-3-031-72114-4_20}{hejrati2024cdal}}

\newcommand{\methodLEAF}{\paperref{LEAF}{https://doi.org/10.1007/978-3-032-04978-0_37}{huang2025leaf}}
\newcommand{\methodColdSegDiff}{\paperref{Cold SegDiff}{https://doi.org/10.1016/j.knosys.2024.112350}{yan2024coldsegdiff}}
\newcommand{\methodRecycledDiff}{\paperref{RecycledDiff}{https://doi.org/10.59275/j.melba.2023-fbe4}{fu2023recycling}}
\newcommand{\methodMambaDiff}{\paperref{MambaDiff}{https://doi.org/10.1109/TIP.2025.3607615}{liu2025mambadiff}}
\newcommand{\methodRRDiff}{\paperref{RR-Diff}{https://doi.org/10.1109/TMI.2024.3494762}{guo2025rrdiff}}

\newtheorem{theorem}{Theorem}[section]
\newtheorem{lemma}[theorem]{Lemma}
\newtheorem{proposition}[theorem]{Proposition}

\newenvironment{proof}{\begin{IEEEproof}}{\end{IEEEproof}}

\title{Rethinking Diffusion Segmentation: When Does It Rely on Its Noisy State, and Does Diffusion Matter?}

\author{Hengzhuo~Yang, Yuming~Zeng, Yuling~Yang
\thanks{(Yuming Zeng and Yuling Yang contributed equally to this work.) (Corresponding author: Hengzhuo Yang.)}
\thanks{H. Yang is with the Department of Mathematics, and Y. Zeng and Y. Yang are
with the Department of Mechanical and Industrial Engineering, Northeastern
University, Boston, MA 02115 USA. (e-mail: yang.hengz@northeastern.edu; zeng.yumi@northeastern.edu; yang.yulin@northeastern.edu)}
}
\date{}
\begin{document}
\bstctlcite{tip:BSTcontrol}
\maketitle

\begin{abstract}

Diffusion models are increasingly adapted from generation to conditional prediction, where a conditioning signal is combined with an evolving noisy representation of the target. In fully supervised segmentation, however, the conditioning image can already support direct target prediction, so endpoint performance alone establishes neither reliance on the added diffusion state nor a deterministic advantage over image-only prediction. For state reliance, we disrupt target-derived state content or correct image-state pairing during retraining of twelve published methods across three datasets, with ten matched seeds per setting. All 40 original-method comparisons whose evaluated-mask routes remained downstream of noised-quantity reconstruction exhibited state reliance, whereas all 30 comparisons with a segmentation-supervised bypass preserved reference performance. Rerouting five originally bypass-capable methods by forcing segmentation supervision through noise-to-mask reconstruction converted all 30 corresponding comparisons from preserved performance to state reliance. For deterministic utility, matched image-only counterparts achieved similar or better performance in 28 of 35 settings overall, including 16 of 20 whose native methods relied on both audited state properties. These results identify supervision path as a determinant of state reliance in the audited methods. Separately, matched image-only counterparts show that diffusion-specific computation often provides no deterministic endpoint advantage, including in methods that rely on the audited state properties. More generally, when conditioning already supports strong target prediction, diffusion-specific claims require additional evidence that the added state is used and that diffusion-specific computation improves the claimed capability beyond a matched condition-only counterpart.

\end{abstract}


\section{Introduction}

Diffusion-based methods have gained increasing attention in fully supervised image segmentation, particularly in medical imaging, where many systems report strong segmentation performance~\cite{amit2021segdiff,wolleb2022ensemdiff,wu2023medsegdiff,chen2023berdiff}. Yet these methods introduce a noisy target state and diffusion timestep into a task for which conventional segmentors already predict the clean segmentation directly from the case-specific image~\cite{ronneberger2015unet,isensee2021nnunet}. Direct clean-target prediction is itself a valid diffusion parameterization~\cite{ho2020ddpm} and can be effective in class-conditional generation~\cite{li2025backtobasics}, but segmentation creates a distinct attribution problem because the conditioning image is spatially aligned with the target mask. In this setting, a diffusion-style predictor can, in principle, obtain accurate segmentations by relying primarily on the image while ignoring the noisy target and timestep, as illustrated in \figref{fig:method_contrast}. Consequently, strong endpoint performance alone establishes neither reliance on the diffusion pathway nor a diffusion-specific advantage over image-only prediction.

\begin{figure}[!t]
\centering
\safeincludegraphics[width=\columnwidth]{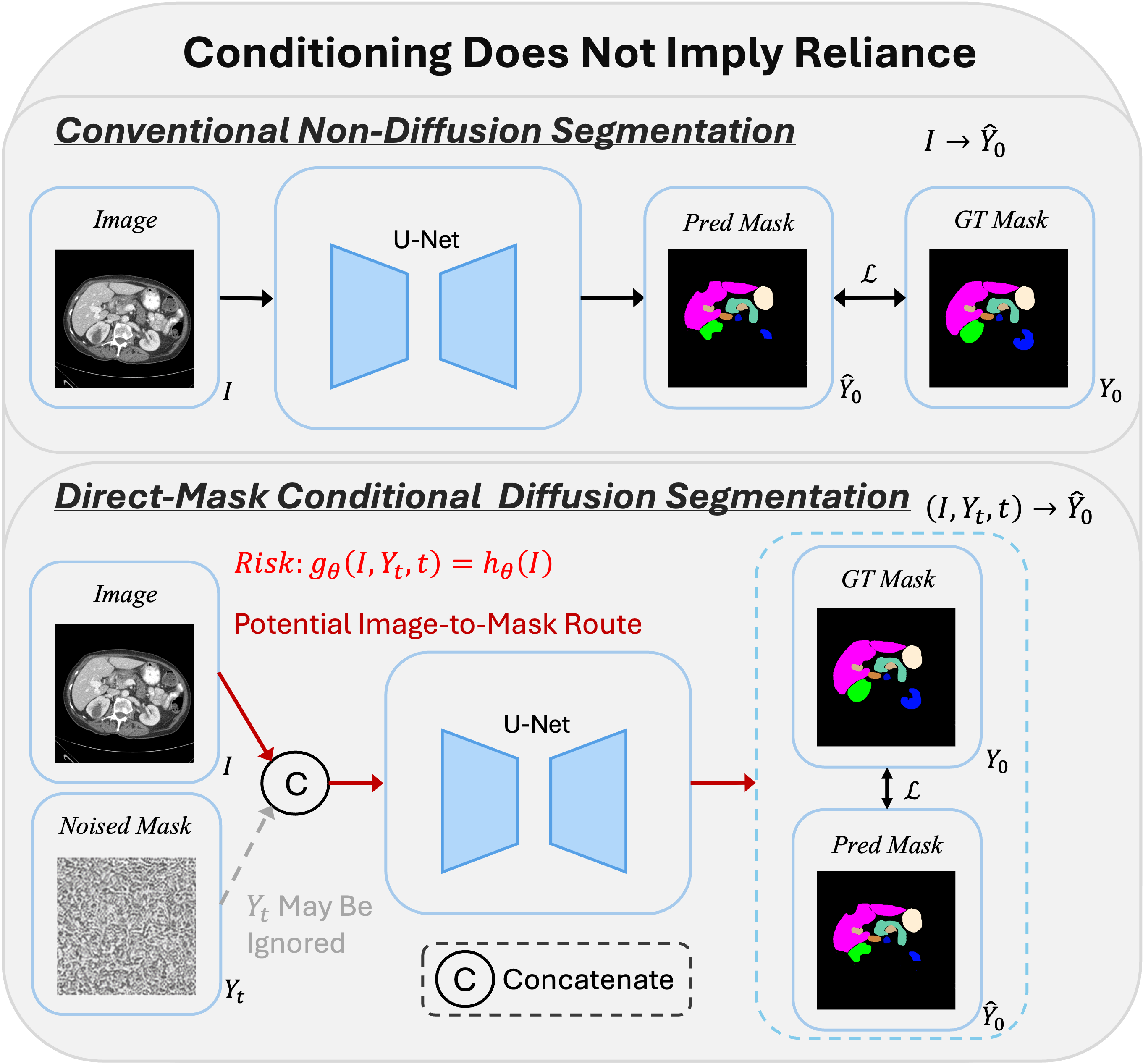}
\caption{Conceptual attribution gap in fully supervised direct-mask diffusion segmentation. Conventional non-diffusion segmentation learns \(\I\!\rightarrow\!\Yhat\), whereas the diffusion-style predictor additionally receives the target-derived noisy state \(\Yt\) and timestep \(t\) during training. The highlighted \(\I\!\rightarrow\!\Yhat\) route denotes a structurally available image-only solution when supported by the supervision path; endpoint performance alone establishes neither learned state reliance nor a deterministic advantage from diffusion-specific computation.}
\label{fig:method_contrast}
\end{figure}

We first ask what determines state reliance. Existing methods differ not only in whether they predict a clean mask, a noised quantity, or both, but also in where segmentation supervision enters and how the evaluated mask is recovered. Some contain a segmentation-supervised route that reaches the evaluated mask without passing through noised-quantity reconstruction, which we refer to as a bypass route. In others, no such bypass exists and the evaluated mask is obtained through noised-quantity reconstruction. Whether such a bypass exists cannot be inferred from the directly predicted quantity or reported loss terms alone. We formalize this structural distinction as Path Category and hypothesize that bypass-capable methods will recover reference performance after retraining under state disruption, whereas methods whose evaluated-mask route requires an accurate noised-quantity prediction will not. We further provide a theoretical analysis of the relationship between supervision path and state reliance.

A fixed-checkpoint contrast provides a preliminary manifestation of this
predicted difference. \methodLEAF{} directly predicts a latent mask, whereas \methodMedSegDiffVOne{} recovers its evaluated mask through a noise-prediction route. Replacing the initial Gaussian state with spatial constants produces much smaller changes for \methodLEAF{} but substantial Dice losses for \methodMedSegDiffVOne{} across three datasets (\figref{fig:motivation_probe}). This contrast suggests path-dependent sensitivity, but a fixed-checkpoint intervention conflates reliance on state information with sensitivity to an altered inference input because the model cannot adapt to the replacement. We therefore move the intervention to training and introduce two retraining-based state audits. The \randomauditfull{} (\randyt{}) replaces the supplied state with independent Gaussian noise, removing target-derived state content while also changing the supplied-state distribution. The
\pairingauditfull{} (\shuffleyt{}) instead constructs the state from another case's mask, retaining method-native mask-derived structure while breaking image-state correspondence.

\begin{figure}[!t]
\centering
\safeincludegraphics[width=\columnwidth]{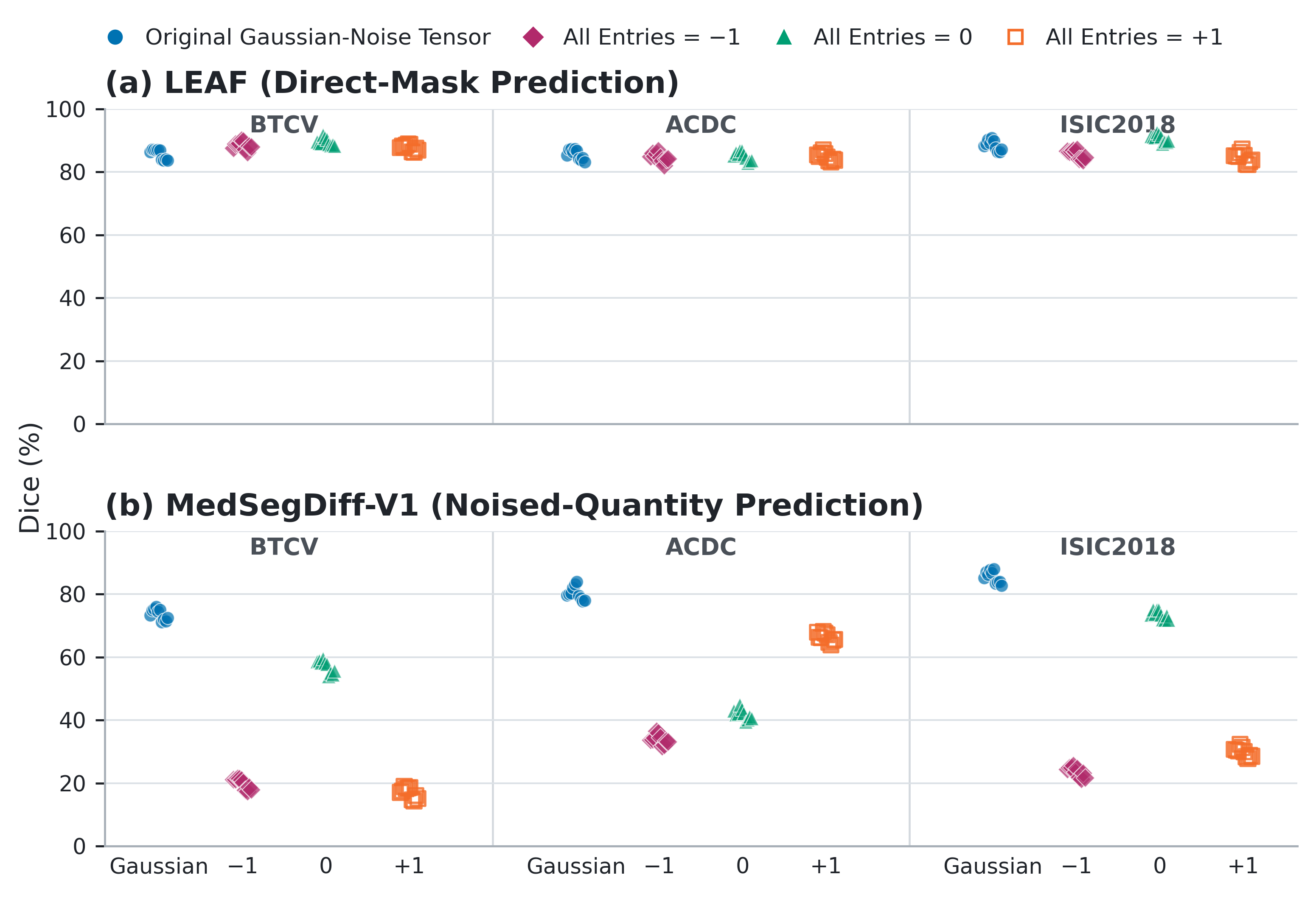}
\caption{Fixed-checkpoint responses of \methodLEAF{} and \methodMedSegDiffVOne{} to initial noisy-state replacement on BTCV~\cite{landman2015btcv}, ACDC~\cite{bernard2018acdc}, and ISIC2018~\cite{codella2019isic2018}. Each point reports test-set Dice from one of ten independently trained checkpoints for the corresponding method-dataset pair. The initial Gaussian state is replaced with spatially constant \(-1\), \(0\), and \(+1\) tensors while all remaining inference operations are held fixed.}
\label{fig:motivation_probe}
\end{figure}

To test whether state reliance changes with supervision path, we reroute five originally bypass-capable methods by requiring segmentation supervision to reach the evaluated mask through noise-to-mask reconstruction and repeat both state audits. If supervision path governs state reliance, rerouting should induce state reliance in methods that were originally bypass-capable.

We separately ask whether diffusion-specific computation provides a deterministic performance advantage over image-only prediction. For each reproduced method, we construct a \corecounterfactualfull{} (\corenodiff{}) by removing noisy-state and timestep inputs and the reverse process while retaining method-compatible image-processing components. For native noised-quantity formulations, the image-only conversion includes compatible changes to the prediction head or supervision. \corenodiff{} then measures the deterministic endpoint performance recoverable from the image alone.

We classify 23 methods by their documented supervision paths and audit twelve methods on BTCV~\cite{landman2015btcv}, ACDC~\cite{bernard2018acdc}, and ISIC2018~\cite{codella2019isic2018}, with ten seed-matched training runs per method-dataset setting. Across the original methods, state-audit responses separated completely by Path Category: every bypass-absent comparison met the state-reliance criterion, whereas every bypass-capable comparison preserved reference performance. Noise-first rerouting reproduced the predicted Preserved-to-SR transition in every modified comparison. The deterministic-utility results followed a different pattern:
\corenodiff{} was Similar or Better in 28 of 35 method-dataset settings,
including 16 of the 20 settings whose native methods relied on both audited state properties. These findings support supervision path as a determinant of state reliance in the audited methods and show that state reliance and deterministic endpoint recoverability are empirically distinct.

\section{Related Work}

\noindent\textbf{Diffusion Models for Image Segmentation.}
Diffusion-based segmentation adapts conditional denoising models by corrupting a representation of the segmentation target and conditioning reverse prediction on the input image. Early methods such as SegDiff, EnsemDiff, and MedSegDiff-V1 used Gaussian diffusion over segmentation masks and iterative reverse refinement~\cite{amit2021segdiff,wolleb2022ensemdiff,wu2023medsegdiff}. Repeated stochastic sampling has served different purposes within this family: EnsemDiff uses multiple diffusion samples as an implicit ensemble, whereas CIMD models ambiguous segmentation by generating multiple plausible masks~\cite{wolleb2022ensemdiff,rahman2023ambiguous}. Subsequent work diversified the state space and corruption process. BerDiff formulated segmentation with Bernoulli diffusion for binary masks~\cite{chen2023berdiff}, while latent approaches such as SDSeg, LDSeg, and TSLDSeg operate on compressed segmentation representations, with SDSeg and LDSeg additionally targeting faster sampling~\cite{lin2024sdseg,zaman2026ldseg,yang2026tsldseg}. These developments establish diffusion segmentation as a family of conditional prediction procedures that differ in target representation, corruption process, and reverse-sampling design.

\noindent\textbf{Prediction Parameterizations and Segmentation Supervision.}
Prediction parameterization is a general design choice in diffusion models. Diffusion formulations admit parameterizations based on the added noise or the clean target~\cite{ho2020ddpm}, and Li and He more recently studied direct clean-data prediction in high-dimensional pixel-space generation, showing that it can be effective under constrained network capacity~\cite{li2025backtobasics}. Both forms appear in diffusion segmentation. SegDiff, EnsemDiff, and MedSegDiff-V1 supervise noise predictions from which clean segmentation estimates are recovered~\cite{amit2021segdiff,wolleb2022ensemdiff,wu2023medsegdiff}. By contrast, cDAL directly predicts and supervises a clean mask, while LEAF applies segmentation supervision to a directly predicted clean-mask latent~\cite{hejrati2024cdal,huang2025leaf}. Several published formulations combine explicit noised-quantity supervision with segmentation-related supervision. UniSegDiff supervises separate noise- and mask-prediction decoders, whereas TSLDSeg applies latent-mask supervision to a clean latent reconstructed from predicted noise~\cite{hu2025unisegdiff,yang2026tsldseg}; MedSegDiff-V2 and SDSeg likewise combine denoising with segmentation-oriented losses~\cite{wu2023medsegdiffv2,lin2024sdseg}. Across these formulations, the predicted quantity, the loss inventory, and where segmentation supervision is applied are related but distinct design choices.

\noindent\textbf{Training-State Construction and Train-Inference Alignment.} 
A separate line of work examines how the noisy segmentation state is constructed during training and how that construction aligns with inference. Standard denoising training forms \(\Yt\) by corrupting the reference mask, whereas inference begins from noise and subsequently operates on model-generated reverse states. Fu et al.\ analyze this mismatch and introduce recycling, in which a segmentation is first predicted for the same image and then passed through the forward process to construct the noisy state used for denoising training~\cite{fu2023aligning,fu2023recycling}. The expanded study evaluates recycling across multiple 2D and 3D medical-imaging datasets and both DDPM and DDIM sampling, reporting improvements over standard diffusion training and related strategies together with more stable inference behavior~\cite{fu2023recycling}. These studies establish training-state construction and train-inference alignment as substantive design dimensions of diffusion segmentation.

\noindent\textbf{Analyses and Non-Diffusion Controls.}
Recent work has also analyzed diffusion segmentation through simplified or non-diffusion counterparts. Fu et al.\ compared their recycling-based direct-mask diffusion model with a matched image-only counterpart that removed the noisy-mask input, timestep, and reverse sampler while retaining a matched U-Net-based framework, reporting broadly comparable deterministic segmentation performance~\cite{fu2023recycling}. Öttl et al.\ analyzed how diffusion segmentation differs from diffusion image generation and retrained EnsemDiff, SegDiff, and MedSegDiff-V1 to directly predict segmentation masks. These variants removed reverse sampling but continued to receive a random tensor at the original noisy-state input~\cite{ottl2024diffseg}.

\section{Mechanism-Attribution Framework}\label{sec:audit}

This section defines the supervision-path taxonomy, audit protocol and decision criteria, introduces the within-method noise-first intervention, and develops the formal rationale for the route-level predictions tested below.

\subsection{Supervision-Path Taxonomy}
\label{sec:prediction_supervision_paths}

We define Path Category by tracing each loss-supported prediction route to the evaluated mask. The primary distinction is whether segmentation-related supervision can reach the evaluated mask through a clean-target route that bypasses noised-quantity reconstruction. Here, the clean target denotes either a pixel-space mask or a deterministic clean-mask representation in latent space. The operational assignment procedure is summarized in \figref{fig:path_category_assignment}.

\begin{figure}[!t]
\centering
\safeincludegraphics[width=\columnwidth]{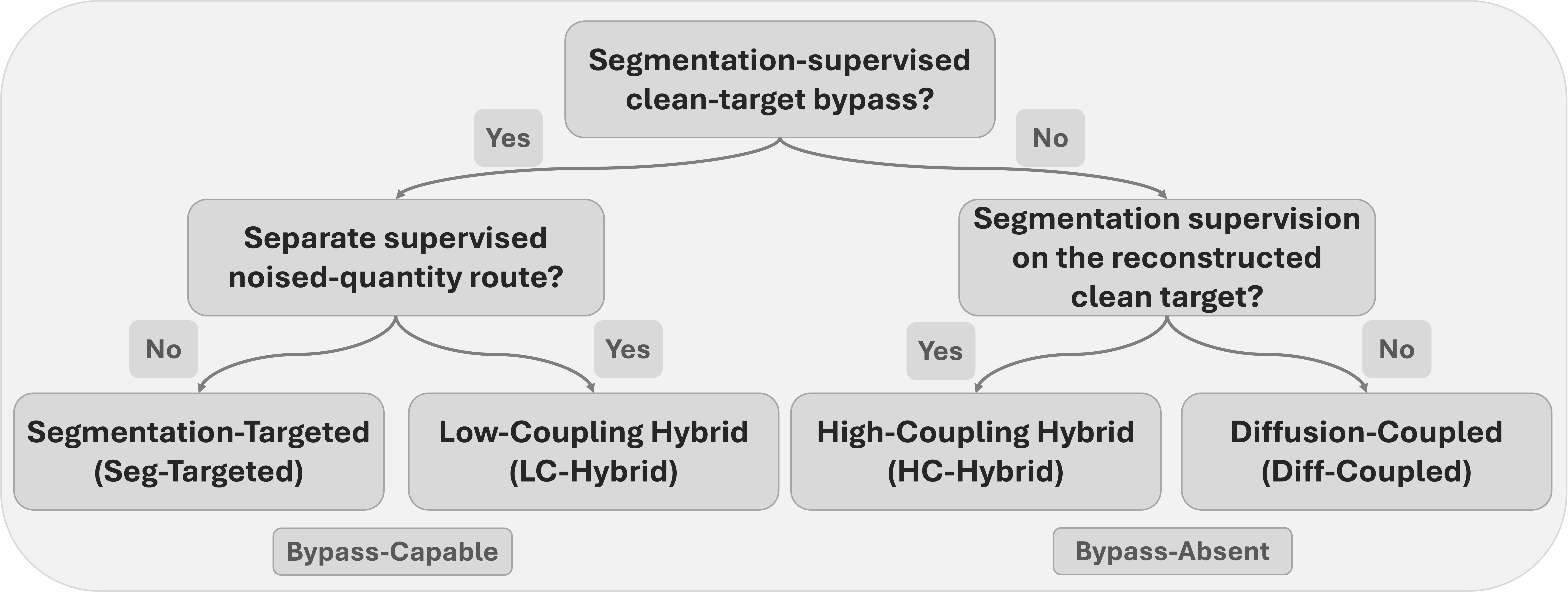}
\caption{Operational assignment of Path Category from documented loss-supported routes to the evaluated mask. Detailed route structures are provided in \figref{fig:prediction_supervision_routes} in Appendix~\ref{app:taxonomy_details} and relative documentation is reported in \tabref{tab:literature_taxonomy}.}
\label{fig:path_category_assignment}
\end{figure}

\noindent\textbf{Segmentation-Targeted (Seg-Targeted).}
Seg-Targeted methods apply segmentation-related supervision to a directly predicted clean target that yields the evaluated mask without passing through noised-quantity reconstruction. They contain no separately supervised noised-quantity route from which the evaluated mask is obtained.

\noindent\textbf{Low-Coupling Hybrid (LC-Hybrid).}
LC-Hybrid methods combine a separately supervised noised-quantity route with a segmentation-supervised clean-target bypass. The evaluated mask can therefore be obtained from the directly predicted clean target without first passing through noised-quantity reconstruction.

\noindent\textbf{High-Coupling Hybrid (HC-Hybrid).}
HC-Hybrid methods contain no segmentation-supervised clean-target bypass. The predictor instead outputs a noised quantity, which is combined with the supplied noisy state through the reconstruction equation to obtain the clean target. Segmentation-related supervision is applied only to the reconstructed clean target.

\noindent\textbf{Diffusion-Coupled (Diff-Coupled).}
Diff-Coupled methods likewise contain no segmentation-supervised clean-target bypass. They supervise a noised quantity from which the evaluated mask is reconstructed, but apply no segmentation-related supervision to the reconstructed clean target.

In the category names, Hybrid denotes the coexistence of a noised-quantity prediction path and segmentation-related supervision. Based on the primary split in \figref{fig:path_category_assignment}, we refer to Seg-Targeted and LC-Hybrid as bypass-capable and to HC-Hybrid and Diff-Coupled as bypass-absent. \tabref{tab:literature_taxonomy} applies these assignment rules to 23 reviewed diffusion-segmentation methods and reports their Loss Families, Directly Predicted Quantities, Path Categories, and classification evidence. Detailed route structures underlying the four categories are provided in \figref{fig:prediction_supervision_routes} in Appendix~\ref{app:taxonomy_details}. The audits introduced next test whether this primary structural distinction corresponds to distinct learned responses to state disruption.

\begin{table*}[!t]
\centering
\small
\setlength{\tabcolsep}{1.5pt}
\renewcommand{\arraystretch}{0.97}
\caption{Supervision-path taxonomy of 23 reviewed diffusion-segmentation methods. Boldface marks the twelve methods evaluated experimentally. Classification Evidence identifies the public source supporting each assignment.}
\label{tab:literature_taxonomy}
\begin{tabularx}{\textwidth}{@{\extracolsep{\fill}}lllll@{}}
\toprule
Method & Loss Family & Directly Predicted Quantity & Path Category & Classification Evidence \\
\midrule
\mbox{\textbf{\methodUniSegDiff}} & Noise + Mask Loss & Noise + Mask & LC-Hybrid & Code + Paper \\
\mbox{\textbf{\methodSSB}} & Noise + Mask Loss & Noise + Mask & LC-Hybrid & Paper only, no code available \\
\midrule
\mbox{\textbf{\methodMedSegDiffVTwo}} & Noise + Mask Loss & Noise & HC-Hybrid & Code + Paper \\
\mbox{\textbf{\methodLDSeg}} & Noise + Mask Loss & Noise & HC-Hybrid & Code + Paper \\
\mbox{\textbf{\methodSDSeg}} & Noise + Latent-Mask Loss & Noise & HC-Hybrid & Code + Paper \\
\mbox{\textbf{\methodTSLDSeg}} & Noise + Latent-Mask Loss & Noise & HC-Hybrid & Code + Paper \\
\midrule
\mbox{\textbf{\methodEnsemDiff}} & Noise + VLB Loss & Noise & Diff-Coupled & Code + Paper \\
\mbox{\textbf{\methodMedSegDiffVOne}} & Noise Loss & Noise & Diff-Coupled & Code + Paper \\
\mbox{\textbf{\methodBerDiff}} & Noise + KL Posterior Loss & Noise & Diff-Coupled & Code + Paper \\
\mbox{\methodACPDiff} & Noise Loss & Noise & Diff-Coupled & Code + Paper \\
\mbox{\methodSegDiff} & Noise Loss & Noise & Diff-Coupled & Code + Paper \\
\mbox{\methodCIMD} & Noise + VLB + Latent-KL Posterior Loss & Noise & Diff-Coupled & Code + Paper \\
\mbox{\methodEIDiffSeg} & Noise Loss & Noise & Diff-Coupled & Code + Paper \\
\mbox{\methodMoDiff} & Noise + morphology + Attention Loss & Noise & Diff-Coupled & Paper only, no code available \\
\mbox{\methodDiffMedSeg} & Noise Loss & Noise & Diff-Coupled & Paper only, no code available \\
\mbox{\methodRetiDiff} & Noise Loss & Noise & Diff-Coupled & Paper only, no code available \\
\midrule
\mbox{\textbf{\methodLEAF}} & Latent-Mask + Distillation Loss & Mask & \pathSegTargeted{} & Code + Paper \\
\mbox{\textbf{\methodCDAL}} & Mask Loss & Mask & \pathSegTargeted{} & Code + Paper \\
\mbox{\textbf{\methodColdSegDiff}} & Mask Loss & Mask & \pathSegTargeted{} & Code + Paper \\
\mbox{\methodRecycledDiff} & Mask Loss & Mask & \pathSegTargeted{} & Code + Paper \\
\mbox{Diff-Unet~\cite{xing2023diffunet}} & Mask + Boundary Loss & Mask & \pathSegTargeted{} & Code + Paper \\
\mbox{\methodMambaDiff} & Mask Loss & Mask & \pathSegTargeted{} & Code + Paper \\
\mbox{\methodRRDiff} & Mask Loss & Mask & \pathSegTargeted{} & Code + Paper \\
\bottomrule
\end{tabularx}
\end{table*}

\subsection{Audit Protocol and Decision Criteria}
\label{sec:audit_protocol}

\figref{fig:three_audits} summarizes three audits, each trained separately and compared with Full, the corresponding reproduced method in its native formulation. For case \(i\), the supplied state is
\begin{equation}
    \mathbf Y_{t_i,i}^{(\mathrm{Full})}
    =
    q_{t_i}^{\mathrm{orig}}
    \left(\Yzero^{(i)};\boldsymbol{\xi}^{(i)}\right).
    \label{eq:audit_full_state}
\end{equation}
Here \(\boldsymbol{\xi}^{(i)}\) collects the method-specific noised-quantity
variables. Gaussian mask diffusion instantiates this state as
\begin{equation}
\begin{aligned}
    \mathbf Y_{t_i,i}^{(\mathrm{Full})}
    & =
    a_{t_i}\Yzero^{(i)}+b_{t_i}\eps^{(i)},
    \\
    a_t
    & =\sqrt{\bar\alpha_t},
    \quad
    b_t=\sqrt{1-\bar\alpha_t},
    \quad
    \eps^{(i)}\sim\mathcal N(\mathbf 0,\mathbf I_{\mathrm{id}}).
\end{aligned}
    \label{eq:audit_gaussian_state}
\end{equation}

In both state audits, the original training tuple and supervision target are first constructed as in Full. Only the state supplied to the predictor is then replaced. \corenodiff{} instead trains a separate image-only counterpart to measure how much segmentation performance can be recovered from image alone.

\begin{figure*}[t!]
\centering
\safeincludegraphics[width=\textwidth]{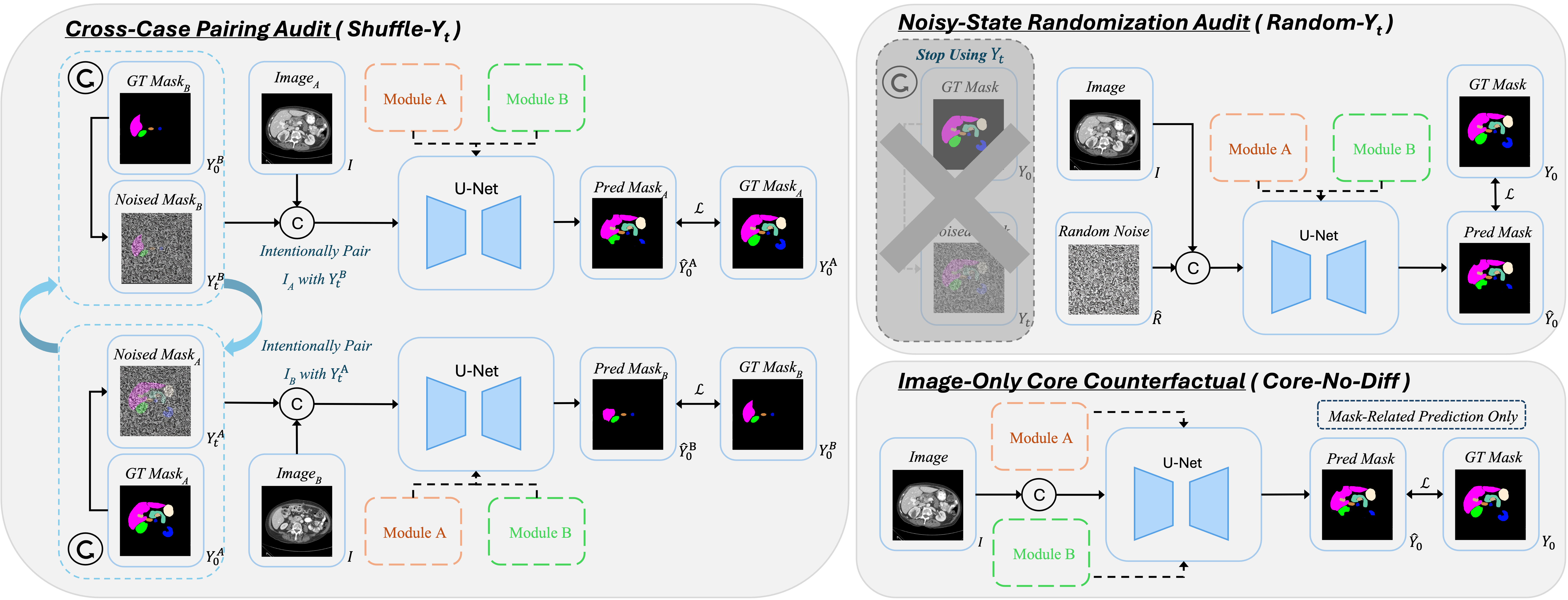}
\caption{Three-part mechanism-attribution protocol. \randyt{} and \shuffleyt{} are separately retrained comparisons that test distinct state properties. \randyt{} replaces the supplied noisy state with independent Gaussian noise, while \shuffleyt{} constructs it from another case's mask. Both retain the image, supervision target and loss, Directly Predicted Quantity, and inference protocol. \corenodiff{} separately trains an image-only capacity counterfactual without diffusion-specific signals or reverse sampling. Mask outputs are schematic.}
\label{fig:three_audits}
\end{figure*}

\subsubsection{\randomauditfull{} (\randyt{})}

The supplied state is replaced with Gaussian noise:
\begin{equation}
    \mathbf Y_{t_i,i}^{(\mathrm{Rand})}
    =
    \etaNoise^{(i)},
    \qquad
    \etaNoise^{(i)}
    \sim\mathcal N(\mathbf 0,\mathbf I_{\mathrm{id}}).
    \label{eq:randyt}
\end{equation}
where \(\etaNoise^{(i)}\) is independent of the original tuple.  The image, timestep convention, Directly Predicted Quantity, original supervision target and loss, architecture, and inference protocol remain unchanged. Random-\(Y_t\) tests whether Full-level performance can be recovered without target-derived information in the supplied state.
\subsubsection{\pairingauditfull{} (\shuffleyt{})}

\shuffleyt{} complements \randyt{} by preserving method-native state structure while breaking its pairing with the current image. Let \(j=\pi(i)\neq i\) be a no-self-match donor assignment. The
supplied state is
\begin{equation}
    \mathbf Y_{t_i,i}^{(\mathrm{Shuf})}
    =
    q_{t_i}^{\mathrm{orig}}
    \left(\Yzero^{(j)};\boldsymbol{\xi}^{(i)}\right),
    \qquad j=\pi(i)\neq i.
    \label{eq:shuffleyt}
\end{equation}
The image, timestep, perturbation variable, Directly Predicted Quantity,
supervision target, and loss remain attached to case \(i\). Donors are
resampled through a no-self-match batch permutation during training. The audit tests whether the method can recover Full-level performance without correct image-state correspondence.

\subsubsection{\corecounterfactualfull{} (\corenodiff{})}

\corenodiff{} removes the noisy state, timestep, and reverse process:
\begin{equation}
    \widehat{\Yzero}_{\mathrm{core}}
    =
    F_{\phi}^{\mathrm{img}}(\I).
    \label{eq:core_no_diff}
\end{equation}
For \pathSegTargeted{} methods, the clean target and compatible mask loss are unchanged. LC-Hybrid conversions retain the direct mask or latent-mask target and compatible loss while removing diffusion-only terms. Diff-Coupled and HC-Hybrid methods require a compatible direct segmentation head and loss. We use Dice with cross-entropy for multiclass tasks and Dice with binary cross-entropy for binary tasks. The resulting comparison assesses whether diffusion-specific computation provides a deterministic segmentation advantage over the corresponding image-only counterpart.

\subsubsection{Decision Criteria}

Every decision is made within a fixed method-dataset setting. Let \(r\) be
Full for original-method audits. The intervention defined in
Section~\ref{sec:noise_first_intervention} uses Modified Full as its reference.
Each condition uses the same ten seed identifiers, and seed is the statistical
unit. For each seed \(s\), let \(\bar m_s(c)\) denote the mean Dice achieved
on the fixed held-out test set by the model trained under condition \(c\),
expressed in percentage points, and define
\begin{equation}
    \Delta_s(c;r)=\bar m_s(c)-\bar m_s(r),
    \qquad
    \delta_r(\kappa)=\kappa\,\mathrm{SD}_r.
    \label{eq:seed_delta}
\end{equation}
where \(\mathrm{SD}_r\) is the sample standard deviation across the ten
reference seeds, computed with denominator \(n-1\). The primary analysis uses
\(\kappa=1.0\), with
\(\kappa\in\{0.5,1.5,2.0\}\) for sensitivity analysis. This margin defines a performance tolerance based on across-seed reference variability.

For each comparison \((c,r)\), the ten matched seeds yield the paired
differences \(\{\Delta_s(c;r)\}_{s=1}^{10}\). Let \(\bar\Delta\) and
\(\mathrm{SD}_\Delta\) denote their mean and sample standard deviation,
respectively. We compute \(\mathrm{SD}_\Delta\) with denominator \(9\), and
define the standard error of the mean difference as
\(\mathrm{SE}_\Delta=\mathrm{SD}_\Delta/\sqrt{10}\). Under a paired \(t\)-analysis of the seed-level differences, the one-sided 95\% bounds are
\begin{equation}
\begin{aligned}
    \mathrm{LCB}_{95}
    & =
    \bar\Delta-t_{0.95,9}\mathrm{SE}_\Delta,
    \\
    \mathrm{UCB}_{95}
    & =
    \bar\Delta+t_{0.95,9}\mathrm{SE}_\Delta.
\end{aligned}
    \label{eq:one_sided_bounds}
\end{equation}
Together, the bounds form a two-sided 90\% confidence interval for the mean paired difference across retraining seeds on the fixed held-out test set. The analysis uses comparison-wise inference without multiplicity adjustment.

For \randyt{} and \shuffleyt{},
\begin{equation}
\begin{cases}
\text{Preserved}, & \mathrm{LCB}_{95}>-\delta_r(\kappa),\\
\text{SR}, & \mathrm{UCB}_{95}<-\delta_r(\kappa),\\
\text{Inconclusive}, & \text{otherwise}.
\end{cases}
\label{eq:yt_decision}
\end{equation}
Preserved is a margin-based non-inferiority decision. We report State Reliance (SR) when the performance loss under an audit is statistically supported to exceed the reference-variability margin, indicating failure to recover Full-level performance after retraining. Under Random-\(Y_t\) and Shuffle-\(Y_t\), this supports reliance on target-derived state content and correct image-state pairing, respectively.

For \corenodiff{},
\begin{equation}
\begin{cases}
\text{Worse}, & \mathrm{UCB}_{95}<-\delta_r(\kappa),\\
\text{Better}, & \mathrm{LCB}_{95}>+\delta_r(\kappa),\\
\text{Similar}, &
\begin{gathered}
\mathrm{LCB}_{95}>-\delta_r(\kappa)\\
\text{and }\mathrm{UCB}_{95}<+\delta_r(\kappa),
\end{gathered}\\
\text{Inconclusive}, & \text{otherwise}.
\end{cases}
\label{eq:core_decision}
\end{equation}
The Similar rule is the confidence-bound form of two one-sided tests at
\(\alpha=0.05\)~\cite{schuirmann1987tost}. These labels describe image-only
capacity and remain separate from SR.

\subsection{Noise-First Supervision-Path Intervention}
\label{sec:noise_first_intervention}

To test whether state reliance changes with supervision path, we reroute five originally bypass-capable methods by requiring segmentation supervision to reach the evaluated mask only through noise-to-mask reconstruction, and repeat \randyt{} and \shuffleyt{}.

For a modified condition \(c\), let \(\widehat{\eps}^{(c)}\) denote the
predictor's noise output from \((\I,\Yt^{(c)},t)\). The corresponding estimate of the clean target \(\Yzero\), reconstructed from the supplied state at timestep \(t\), is
\begin{equation}
    \widehat{\Yzero}^{\,\mathrm{rec},(c)}
    =
    \frac{\Yt^{(c)}-b_t\widehat{\eps}^{(c)}}{a_t}.
    \label{eq:noise_first_reconstruction}
\end{equation}
The noise loss is evaluated on \(\widehat{\eps}^{(c)}\), and any mask or
latent-mask loss is evaluated only on the resulting reconstruction. Latent
variants apply the same relation in latent-mask space before decoding. This construction instantiates an HC-Hybrid route in \figref{fig:prediction_supervision_routes} with explicit noise supervision and segmentation supervision applied after reconstruction.

For \methodUniSegDiff{} and \methodSSB{}, both originally LC-Hybrid, the intervention relocates existing segmentation supervision from the segmentation-supervised bypass to the reconstruction, providing a direct test of bypass removal. For \methodLEAF{}, \methodColdSegDiff{}, and \methodCDAL{}, all originally \pathSegTargeted{}, the redesign also changes the Directly Predicted Quantity from Mask to Noise, yielding broader noise-first modifications. Accordingly, the five modified variants are treated as HC-Hybrid in the intervention analyses, while \tabref{tab:literature_taxonomy} retains the original method assignments.

Each state audit of a modified variant is evaluated against that variant's Modified Full reference. The intervention test is whether the noise-first rerouting produces a Preserved-to-SR transition under both state audits. Original Full and Modified Full are compared descriptively to report the accompanying change in deterministic performance.

\subsection{Formal Analysis of State Audits and Supervision Paths}
\label{sec:formal_rationale}

Formally, the route-level hypothesis rests on whether an audit-invariant image-only endpoint remains representable and is supported by the documented supervision path. For routes with explicit noised-quantity supervision, we additionally analyze recoverability of that supervised target and its compatibility with clean-target reconstruction under state disruption. Separately, the reconstruction identity characterizes the supplied-state-reliant intermediate required by a reconstruction-coupled route.
Throughout this analysis, \(\Yzero\) denotes the clean reference mask for the
sampled case; \(t\) indexes the training-time supplied state and the corresponding timestep-conditioned predictor call or reconstruction.

Let \(\tau_m(\Yzero)\) denote the segmentation target used in this analysis,
either the clean mask in pixel space or a deterministic encoding of that mask
in a latent segmentation space. A conventional non-diffusion segmentor
minimizes the segmentation risk
\(\mathbb E[\ell_m(h(\I),\tau_m(\Yzero))]\) over image-only predictors
\(h(\I)\). Under the realizable abstraction \(\Yzero=\phi(\I)\) almost surely,
\(h_0=\tau_m\circ\phi\) attains zero segmentation risk when the loss is
nonnegative and vanishes at the correct target
(Proposition~\ref{prop:app_image_only_optimum}).

Relative to this non-diffusion reference, suppose the diffusion-style
direct-mask predictor class contains a common image-only realization satisfying
\[
    g_h(\I,\Yt^{(c)},t)=h(\I),
    \qquad
    c\in\{\mathrm{Full},\mathrm{Rand},\mathrm{Shuf}\}.
\]
Preservation of the joint law of \((\I,\Yzero)\) then gives
\begin{equation}
    \mathcal R_m^{(c)}(g_h)
    =
    \mathbb E\!\left[
      \ell_m\!\left(h(\I),\tau_m(\Yzero)\right)
    \right]
    \label{eq:audit_invariant_segmentation_risk}
\end{equation}
under every condition
(Lemma~\ref{lem:app_image_only_invariance}).
The availability of this solution is category-neutral.

Path Category then determines whether the documented objective provides a segmentation-supervised route that can support the same image-only solution without passing through noised-quantity reconstruction. Seg-Targeted and the LC-Hybrid bypass provide such a route. HC-Hybrid and Diff-Coupled provide no corresponding bypass, although their complete routes may still represent the same solution.

For a pixel-space clean-mask predictor
\(h:\mathcal I\to\mathbb R^d\), solving the Gaussian reconstruction for the
intermediate output shows that realizing \(h(\I)\) from a supplied state
\(\mathbf y\) requires
\(\psi(\I,\mathbf y,t)=(\mathbf y-a_t h(\I))/b_t\).
More importantly, for fixed \((\I,t)\), if two distinct supplied states
\(\mathbf y\neq\widetilde{\mathbf y}\) both reconstruct to the same endpoint
\(h(\I)\), then
\begin{equation}
    \psi(\I,\mathbf y,t)
    -
    \psi(\I,\widetilde{\mathbf y},t)
    =
    \frac{\mathbf y-\widetilde{\mathbf y}}{b_t}
    \neq \mathbf 0.
    \label{eq:image_only_endpoint_through_noise}
\end{equation}
Thus, exact realization of the same image-only endpoint across distinct
supplied states requires a supplied-state-dependent intermediate output; an
intermediate of the form \(v(\I,t)\) cannot realize \(h(\I)\) at both states.

Even when the image determines \(\Yzero\), it does not determine the
independently sampled realization of \(\eps\). Let
\(X_c=(\I,\Yt^{(c)},t)\), and define
\begin{equation}
\begin{aligned}
    \mathcal R_{I,t}
    &:={}
    \inf_f
    \mathbb E\!\left\|f(\I,t)-\eps\right\|^2,\\
    \mathcal R_c
    &:={}
    \inf_{f_c}
    \mathbb E\!\left\|f_c(X_c)-\eps\right\|^2,\\
    \mathcal R_{\mathrm{Full}}=0
    &\le
    \mathcal R_{\mathrm{Shuf}}
    <
    \mathcal R_{\mathrm{Rand}}
    =\mathcal R_{I,t}=d.
\end{aligned}
    \label{eq:audit_noise_risks}
\end{equation}
Here the infima range over square-integrable predictors, and \(d\) is the
noise dimension under unnormalized squared error.
Propositions~\ref{prop:app_full_information}
and~\ref{prop:app_gaussian_risks} establish this ordering. In particular,
under the realizable abstraction, the matched Full inputs determine the
realized noise exactly as
\(\eps=(\Yt^{(\mathrm{Full})}-a_t\phi(\I))/b_t\).
Random adds no information about the retained noise beyond \((\I,t)\).
Shuffle retains information because its state contains the current \(\eps\),
while uncertainty about the donor mask can leave irreducible prediction error.

The final step connects the retained noise target to clean-target recovery.
For a noise prediction \(\widehat{\eps}^{(c)}\), subtracting the Full relation
from the reconstruction under condition \(c\) gives
\begin{equation}
    \widehat{\Yzero}^{\,\mathrm{rec},(c)}-\Yzero
    =
    \frac{\Yt^{(c)}-\Yt^{(\mathrm{Full})}}{a_t}
    +
    \frac{b_t}{a_t}
    \left(\eps-\widehat{\eps}^{(c)}\right).
    \label{eq:audit_reconstruction_error}
\end{equation}
Equation~\eqref{eq:audit_reconstruction_error} makes the compatibility
condition explicit:
\begin{equation}
    \widehat{\eps}^{(c)}=\eps
    \quad\text{and}\quad
    \widehat{\Yzero}^{\,\mathrm{rec},(c)}=\Yzero
    \quad\Longleftrightarrow\quad
    \Yt^{(c)}=\Yt^{(\mathrm{Full})}.
    \label{eq:noise_mask_compatibility}
\end{equation}
In Full, the state-mismatch term vanishes, so accurate retained-noise prediction is compatible with accurate reconstruction of \(\Yzero\). Under an audit, exact retained-noise prediction instead leaves the state mismatch in
Eq.~\eqref{eq:audit_reconstruction_error}.

For Shuffle,
\(\Yt^{(\mathrm{Shuf})}=a_t\Yzero'+b_t\eps\), so exact retained-noise
prediction reconstructs the donor mask:
\[
    \widehat{\eps}^{(\mathrm{Shuf})}=\eps
    \quad\Longrightarrow\quad
    \widehat{\Yzero}^{\,\mathrm{rec},(\mathrm{Shuf})}=\Yzero'.
\]
When \(\Yzero'\neq\Yzero\), exact reconstruction of the clean target instead
requires the shifted intermediate
\(\widehat{\eps}^{(\mathrm{Shuf})}
=\eps+(a_t/b_t)(\Yzero'-\Yzero)\).
Random produces the corresponding incompatibility with its independent
Gaussian replacement. Proposition~\ref{prop:app_reconstruction_compatibility}
gives the general identity and the condition-specific derivations.

Applied to the four Path Categories, these results give the route-level predictions. In LC-Hybrid and \pathSegTargeted{}, segmentation supervision can directly support the audit-invariant candidate \(g_h=h(\I)\) when that candidate is representable. For Diff-Coupled and HC-Hybrid methods within the Gaussian reconstruction setting, no corresponding segmentation-supervised bypass to the evaluated mask is available, and the retained-noise target is no longer exactly recoverable under state disruption. HC-Hybrid additionally imposes reconstruction-based segmentation supervision, for which retained-noise correctness and clean-target correctness are incompatible under nonzero state mismatch. We therefore predict Preserved responses for LC-Hybrid and \pathSegTargeted{}, and SR responses for Diff-Coupled and HC-Hybrid.

The route-level and intervention hypotheses are evaluated through finite
retraining and the Preserved/SR decision rules. The formal derivations directly apply to the Gaussian \(\epsilon\)-prediction reconstruction considered here. Appendix~\ref{app:theory_experiment_boundary} gives the complete assumptions, method-specific boundaries, and latent-state extension.

\section{Experiments and Results}\label{sec:experiments} 

We report results for twelve methods across BTCV, ACDC, and ISIC2018, with ten seed-matched training runs per condition. Implementation details are provided in Appendix~\ref{app:experimental_setup}.

\subsection{State-Audit Responses Separate by Path Category} 

Original-method responses followed the prespecified Path Category predictions under both \randyt{} and \shuffleyt{}. \tabref{tab:main_all_datasets} reports absolute Dice scores, while \tabref{tab:stat_all_datasets} reports the paired changes and decisions at the primary reference-variability margin.

\begin{table*}[!t]
\centering
\small
\renewcommand{\arraystretch}{0.97}
\caption{Original-method performance on BTCV, ACDC, and ISIC2018. Dice (\%) is reported as mean \(\pm\) SD over ten matched seeds. Full is the reproduced reference. BerDiff/ACDC was unavailable.}
\label{tab:main_all_datasets}
\begin{tabular*}{\textwidth}{@{\extracolsep{\fill}}cll|c|c|c|c@{}}
\toprule
Dataset & Method & Path Category & Full Dice & Random-$Y_t$ Dice & Shuffle-$Y_t$ Dice & Core-No-Diff Dice \\
\midrule
\multirow[c]{12}{*}{\textbf{BTCV}} & \methodUniSegDiff & LC-Hybrid & 85.82$\pm$1.86 & 85.78$\pm$1.94 & 85.80$\pm$1.88 & 85.86$\pm$1.99 \\
& \methodSSB & LC-Hybrid & 84.01$\pm$1.70 & 83.92$\pm$1.74 & 83.94$\pm$1.71 & 83.82$\pm$1.82 \\
& \methodMedSegDiffVTwo & HC-Hybrid & 81.19$\pm$1.75 & 12.24$\pm$2.01 & 11.17$\pm$2.03 & 81.09$\pm$1.82 \\
& \methodSDSeg & HC-Hybrid & 90.56$\pm$1.57 & 77.69$\pm$1.59 & 83.43$\pm$1.57 & 90.55$\pm$1.61 \\
& \methodTSLDSeg & HC-Hybrid & 82.01$\pm$1.60 & 14.28$\pm$1.63 & 19.91$\pm$1.61 & 83.98$\pm$1.63 \\
& \methodLDSeg & HC-Hybrid & 64.26$\pm$2.06 & 19.31$\pm$2.06 & 50.16$\pm$2.06 & 75.67$\pm$2.09 \\
& \methodEnsemDiff & Diff-Coupled & 90.15$\pm$1.60 & 3.95$\pm$1.64 & 24.23$\pm$1.61 & 87.92$\pm$1.68 \\
& \methodMedSegDiffVOne & Diff-Coupled & 73.59$\pm$1.78 & 14.27$\pm$1.83 & 15.05$\pm$1.83 & 78.20$\pm$1.85 \\
& \methodBerDiff & Diff-Coupled & 88.24$\pm$1.55 & 14.92$\pm$1.59 & 34.70$\pm$1.57 & 88.26$\pm$1.59 \\
& \methodLEAF & \pathSegTargeted{} & 85.65$\pm$1.62 & 85.91$\pm$1.62 & 85.71$\pm$1.73 & 85.60$\pm$1.61 \\
& \methodColdSegDiff & \pathSegTargeted{} & 79.40$\pm$2.11 & 82.96$\pm$1.69 & 82.75$\pm$1.70 & 79.98$\pm$1.70 \\
& \methodCDAL & \pathSegTargeted{} & 81.00$\pm$1.54 & 82.64$\pm$1.59 & 81.96$\pm$1.55 & 77.64$\pm$1.57 \\
\cmidrule(lr){1-7}
\multirow[c]{12}{*}{\textbf{ACDC}} & \methodUniSegDiff & LC-Hybrid & 82.76$\pm$1.72 & 82.70$\pm$1.66 & 82.72$\pm$1.80 & 83.10$\pm$1.66 \\
& \methodSSB & LC-Hybrid & 81.29$\pm$1.77 & 81.29$\pm$1.76 & 81.19$\pm$1.69 & 81.13$\pm$1.65 \\
& \methodMedSegDiffVTwo & HC-Hybrid & 72.25$\pm$1.54 & 7.65$\pm$1.81 & 7.33$\pm$1.78 & 72.31$\pm$1.55 \\
& \methodSDSeg & HC-Hybrid & 80.85$\pm$1.75 & 3.07$\pm$1.78 & 4.37$\pm$1.81 & 86.90$\pm$1.82 \\
& \methodTSLDSeg & HC-Hybrid & 77.03$\pm$1.69 & 2.87$\pm$1.74 & 4.00$\pm$1.73 & 85.58$\pm$1.74 \\
& \methodLDSeg & HC-Hybrid & 62.65$\pm$1.60 & 27.32$\pm$1.64 & 39.25$\pm$1.64 & 80.04$\pm$1.65 \\
& \methodEnsemDiff & Diff-Coupled & 89.77$\pm$1.97 & 2.96$\pm$1.58 & 18.36$\pm$1.97 & 89.80$\pm$2.02 \\
& \methodMedSegDiffVOne & Diff-Coupled & 80.28$\pm$2.11 & 3.04$\pm$1.68 & 13.63$\pm$2.14 & 70.56$\pm$1.89 \\
& \methodBerDiff & Diff-Coupled & - & - & - & - \\
& \methodLEAF & \pathSegTargeted{} & 85.55$\pm$1.59 & 86.85$\pm$1.62 & 86.85$\pm$1.62 & 87.12$\pm$1.63 \\
& \methodColdSegDiff & \pathSegTargeted{} & 83.51$\pm$2.13 & 84.69$\pm$2.16 & 84.50$\pm$2.16 & 83.58$\pm$2.15 \\
& \methodCDAL & \pathSegTargeted{} & 87.63$\pm$1.76 & 87.97$\pm$1.78 & 87.57$\pm$1.81 & 87.60$\pm$1.79 \\
\cmidrule(lr){1-7}
\multirow[c]{12}{*}{\textbf{ISIC2018}} & \methodUniSegDiff & LC-Hybrid & 89.27$\pm$1.72 & 89.33$\pm$1.73 & 89.32$\pm$1.73 & 89.40$\pm$1.79 \\
& \methodSSB & LC-Hybrid & 87.07$\pm$2.04 & 87.09$\pm$2.02 & 87.02$\pm$2.01 & 87.11$\pm$2.00 \\
& \methodMedSegDiffVTwo & HC-Hybrid & 83.52$\pm$1.72 & 44.09$\pm$1.53 & 41.67$\pm$1.83 & 85.72$\pm$1.77 \\
& \methodSDSeg & HC-Hybrid & 89.13$\pm$1.53 & 31.34$\pm$1.57 & 57.65$\pm$1.54 & 90.87$\pm$1.56 \\
& \methodTSLDSeg & HC-Hybrid & 90.30$\pm$1.74 & 30.75$\pm$1.79 & 52.40$\pm$1.79 & 90.39$\pm$1.78 \\
& \methodLDSeg & HC-Hybrid & 85.70$\pm$1.77 & 78.97$\pm$1.82 & 81.22$\pm$1.84 & 88.90$\pm$1.80 \\
& \methodEnsemDiff & Diff-Coupled & 89.20$\pm$1.92 & 11.21$\pm$1.93 & 47.86$\pm$1.93 & 87.40$\pm$1.91 \\
& \methodMedSegDiffVOne & Diff-Coupled & 85.44$\pm$1.92 & 41.17$\pm$1.94 & 39.26$\pm$1.93 & 77.88$\pm$1.96 \\
& \methodBerDiff & Diff-Coupled & 90.17$\pm$1.54 & 32.81$\pm$1.57 & 64.21$\pm$1.57 & 90.10$\pm$1.59 \\
& \methodLEAF & \pathSegTargeted{} & 88.44$\pm$1.58 & 88.99$\pm$1.57 & 88.97$\pm$1.57 & 88.52$\pm$1.59 \\
& \methodColdSegDiff & \pathSegTargeted{} & 87.12$\pm$1.64 & 88.34$\pm$1.65 & 89.02$\pm$1.67 & 88.77$\pm$1.67 \\
& \methodCDAL & \pathSegTargeted{} & 88.47$\pm$1.76 & 88.81$\pm$1.85 & 88.72$\pm$1.85 & 88.42$\pm$1.82 \\
\bottomrule
\end{tabular*}
\end{table*}

\begin{table*}[!t]
\centering
\small
\setlength{\tabcolsep}{1.5pt}
\renewcommand{\arraystretch}{0.94}
\caption{Original-method decisions at \(\kappa=1.0\) from ten seed-matched pairs. \(\Delta\)Dice is condition minus Full, and \(\delta_r\) is the seed-level SD of Full. Mean paired \(\Delta\)Dice is reported as \(\bar{\Delta}\pm t_{0.95,9}\mathrm{SE}_{\Delta}\), corresponding to the 90\% CI \([\mathrm{LCB}_{95},\mathrm{UCB}_{95}]\). SR denotes supported state reliance. BerDiff/ACDC was unavailable. pp is percentage points.}
\label{tab:stat_all_datasets}
\begin{tabular*}{\textwidth}{@{\extracolsep{\fill}}l@{\hspace{2pt}}l|c|cc|cc|cc@{}}
\toprule
\multirow{2}{*}{Method} & \multirow{2}{*}{Path Category} & \multirow{2}{*}{\(\delta_r\) (pp)} & \multicolumn{2}{c|}{Random-\(Y_t\)} & \multicolumn{2}{c|}{Shuffle-\(Y_t\)} & \multicolumn{2}{c}{Core-No-Diff} \\
& & & \makecell{Mean \(\Delta\)Dice (pp)} & Decision & \makecell{Mean \(\Delta\)Dice (pp)} & Decision & \makecell{Mean \(\Delta\)Dice (pp)} & Decision \\
\midrule
\multicolumn{9}{@{}l}{\textbf{BTCV}} \\[-0.3ex]
\methodUniSegDiff & LC-Hybrid & 1.86 & $-0.04\pm0.10$ & Preserved & $-0.02\pm0.09$ & Preserved & $0.04\pm0.49$ & Similar \\
\methodSSB & LC-Hybrid & 1.70 & $-0.09\pm0.07$ & Preserved & $-0.07\pm0.08$ & Preserved & $-0.19\pm0.14$ & Similar \\
\methodMedSegDiffVTwo & HC-Hybrid & 1.75 & $-68.94\pm0.62$ & SR & $-70.02\pm0.78$ & SR & $-0.10\pm0.13$ & Similar \\
\methodSDSeg & HC-Hybrid & 1.57 & $-12.87\pm0.06$ & SR & $-7.13\pm0.06$ & SR & $0.00\pm0.08$ & Similar \\
\methodTSLDSeg & HC-Hybrid & 1.60 & $-67.73\pm0.07$ & SR & $-62.11\pm0.06$ & SR & $1.97\pm0.09$ & Better \\
\methodLDSeg & HC-Hybrid & 2.06 & $-44.95\pm0.06$ & SR & $-14.10\pm0.07$ & SR & $11.41\pm0.08$ & Better \\
\methodEnsemDiff & Diff-Coupled & 1.60 & $-86.19\pm0.07$ & SR & $-65.92\pm0.05$ & SR & $-2.23\pm0.09$ & Worse \\
\methodMedSegDiffVOne & Diff-Coupled & 1.78 & $-59.32\pm0.06$ & SR & $-58.53\pm0.07$ & SR & $4.61\pm0.07$ & Better \\
\methodBerDiff & Diff-Coupled & 1.55 & $-73.32\pm0.07$ & SR & $-53.54\pm0.05$ & SR & $0.02\pm0.07$ & Similar \\
\methodLEAF & \pathSegTargeted{} & 1.62 & $0.26\pm0.20$ & Preserved & $0.07\pm0.47$ & Preserved & $-0.05\pm0.08$ & Similar \\
\methodColdSegDiff & \pathSegTargeted{} & 2.11 & $3.56\pm0.65$ & Preserved & $3.35\pm0.65$ & Preserved & $0.58\pm0.64$ & Similar \\
\methodCDAL & \pathSegTargeted{} & 1.54 & $1.64\pm0.06$ & Preserved & $0.96\pm0.05$ & Preserved & $-3.36\pm0.07$ & Worse \\
\cmidrule(lr){1-9}
\multicolumn{9}{@{}l}{\textbf{ACDC}} \\[-0.3ex]
\methodUniSegDiff & LC-Hybrid & 1.72 & $-0.06\pm0.11$ & Preserved & $-0.04\pm0.18$ & Preserved & $0.34\pm0.24$ & Similar \\
\methodSSB & LC-Hybrid & 1.77 & $-0.01\pm0.06$ & Preserved & $-0.10\pm0.12$ & Preserved & $-0.17\pm0.26$ & Similar \\
\methodMedSegDiffVTwo & HC-Hybrid & 1.54 & $-64.59\pm0.56$ & SR & $-64.92\pm0.57$ & SR & $0.06\pm0.08$ & Similar \\
\methodSDSeg & HC-Hybrid & 1.75 & $-77.78\pm0.08$ & SR & $-76.47\pm0.07$ & SR & $6.05\pm0.07$ & Better \\
\methodTSLDSeg & HC-Hybrid & 1.69 & $-74.16\pm0.07$ & SR & $-73.04\pm0.07$ & SR & $8.55\pm0.07$ & Better \\
\methodLDSeg & HC-Hybrid & 1.60 & $-35.33\pm0.05$ & SR & $-23.40\pm0.07$ & SR & $17.39\pm0.06$ & Better \\
\methodEnsemDiff & Diff-Coupled & 1.97 & $-86.81\pm0.40$ & SR & $-71.41\pm0.06$ & SR & $0.03\pm0.09$ & Similar \\
\methodMedSegDiffVOne & Diff-Coupled & 2.11 & $-77.24\pm0.85$ & SR & $-66.65\pm0.51$ & SR & $-9.72\pm0.88$ & Worse \\
\methodBerDiff & Diff-Coupled & - & - & - & - & - & - & - \\
\methodLEAF & \pathSegTargeted{} & 1.59 & $1.30\pm0.07$ & Preserved & $1.30\pm0.06$ & Preserved & $1.58\pm0.11$ & Inconclusive \\
\methodColdSegDiff & \pathSegTargeted{} & 2.13 & $1.18\pm0.05$ & Preserved & $0.99\pm0.05$ & Preserved & $0.07\pm0.09$ & Similar \\
\methodCDAL & \pathSegTargeted{} & 1.76 & $0.34\pm0.15$ & Preserved & $-0.06\pm0.37$ & Preserved & $-0.03\pm0.06$ & Similar \\
\cmidrule(lr){1-9}
\multicolumn{9}{@{}l}{\textbf{ISIC2018}} \\[-0.3ex]
\methodUniSegDiff & LC-Hybrid & 1.72 & $0.06\pm0.16$ & Preserved & $0.05\pm0.16$ & Preserved & $0.13\pm0.17$ & Similar \\
\methodSSB & LC-Hybrid & 2.04 & $0.02\pm0.09$ & Preserved & $-0.05\pm0.18$ & Preserved & $0.04\pm0.12$ & Similar \\
\methodMedSegDiffVTwo & HC-Hybrid & 1.72 & $-39.43\pm0.38$ & SR & $-41.86\pm0.69$ & SR & $2.19\pm0.07$ & Better \\
\methodSDSeg & HC-Hybrid & 1.53 & $-57.79\pm0.08$ & SR & $-31.48\pm0.08$ & SR & $1.75\pm0.10$ & Better \\
\methodTSLDSeg & HC-Hybrid & 1.74 & $-59.55\pm0.08$ & SR & $-37.90\pm0.07$ & SR & $0.09\pm0.12$ & Similar \\
\methodLDSeg & HC-Hybrid & 1.77 & $-6.73\pm0.07$ & SR & $-4.48\pm0.07$ & SR & $3.20\pm0.07$ & Better \\
\methodEnsemDiff & Diff-Coupled & 1.92 & $-77.99\pm0.07$ & SR & $-41.34\pm0.06$ & SR & $-1.81\pm0.18$ & Inconclusive \\
\methodMedSegDiffVOne & Diff-Coupled & 1.92 & $-44.27\pm0.07$ & SR & $-46.18\pm0.06$ & SR & $-7.57\pm0.06$ & Worse \\
\methodBerDiff & Diff-Coupled & 1.54 & $-57.37\pm0.06$ & SR & $-25.96\pm0.07$ & SR & $-0.08\pm0.12$ & Similar \\
\methodLEAF & \pathSegTargeted{} & 1.58 & $0.54\pm0.08$ & Preserved & $0.53\pm0.08$ & Preserved & $0.08\pm0.07$ & Similar \\
\methodColdSegDiff & \pathSegTargeted{} & 1.64 & $1.22\pm0.07$ & Preserved & $1.90\pm0.06$ & Preserved & $1.65\pm0.08$ & Inconclusive \\
\methodCDAL & \pathSegTargeted{} & 1.76 & $0.34\pm0.21$ & Preserved & $0.25\pm0.28$ & Preserved & $-0.05\pm0.08$ & Similar \\
\bottomrule
\end{tabular*}
\end{table*}

Across the 20 available method-dataset settings assigned to HC-Hybrid or Diff-Coupled, all 40 state-audit comparisons met the SR criterion (\tabref{tab:stat_all_datasets}). Even the smallest losses occurred for \methodLDSeg{}/ISIC2018: mean \(\Delta\)Dice was \(-6.73\) pp (90\% CI, \([-6.81,-6.66]\)) under \randyt{} and \(-4.48\) pp (90\% CI, \([-4.55,-4.41]\)) under \shuffleyt{}, against a reference margin of \(\delta_r=1.77\) pp. The \randyt{} results support reliance on target-derived state content under the matched retraining counterfactual, which also includes the Gaussian replacement shift. The concordant \shuffleyt{} results separately support reliance on correct image-state pairing while retaining method-native mask-derived structure.

The bypass-capable group showed the complementary response. Across the 15 LC-Hybrid or \pathSegTargeted{} method-dataset settings, all 30 state-audit comparisons met the Preserved criterion. The most negative mean changes were limited to \(-0.09\) pp under \randyt{} and \(-0.10\) pp under \shuffleyt{}. These decisions support Full-level recoverability after retraining under either state disruption.

The response split was not explained by the reported loss inventory. \methodUniSegDiff{}, \methodSSB{}, \methodMedSegDiffVTwo{}, and LDSeg all belong to the Noise + Mask Loss family. Nevertheless, the two LC-Hybrid methods produced 12 Preserved decisions, whereas the two HC-Hybrid methods produced 12 SR decisions. \methodBerDiff{} also met the SR criterion under both audits despite its Bernoulli formulation. Thus, across the available original-method settings, state-audit response separated completely by Path Category. The corresponding outcomes for \methodLEAF{} and \methodMedSegDiffVOne{} were directionally consistent with the motivating fixed-checkpoint contrast in \figref{fig:motivation_probe}. 

\noindent\textit{Noise-loss removal ablation.}
To further test whether the HC-Hybrid response was attributable to the explicit noise-loss term, we removed this term from the four audited HC-Hybrid methods while retaining their reconstruction-based segmentation supervision. All 24 method-dataset audit comparisons remained SR under \randyt{} and \shuffleyt{} (\tabref{tab:hc_no_noise_perf} and \tabref{tab:hc_no_noise_stat}), showing that, in these audited HC-Hybrid variants, the SR response persists without an explicit noise-loss term.

A secondary magnitude pattern accompanied this categorical separation. Among the 20 state-reliant original method-dataset settings, \shuffleyt{} produced a smaller absolute Dice loss than \randyt{} in 16 settings, directionally consistent with the retained-noise recoverability ordering \(\mathcal{R}_{\mathrm{Shuf}} < \mathcal{R}_{\mathrm{Rand}}\) derived for the analyzed Gaussian noise-prediction setting in Section~\ref{sec:formal_rationale}. Because the two audits are not calibrated to equal intervention strength, we treat this correspondence as descriptive rather than as an ordering of reliance magnitude.

\subsection{Noise-First Rerouting Induces State Reliance}

Noise-first rerouting produced the predicted Preserved-to-SR transition across all five originally bypass-capable methods. \tabref{tab:fix_perf} reports their absolute scores, and \tabref{tab:fix_stat} reports paired audit decisions relative to each variant's Modified Full.

\begin{table*}[!t]
\centering
\small
\renewcommand{\arraystretch}{0.97}
\caption{Noise-first intervention performance on BTCV, ACDC, and ISIC2018. Dice (\%) is reported as mean \(\pm\) SD over ten matched seeds. Original Full and Modified Full are the unmodified and rerouted references. All rerouted variants are treated as HC-Hybrid regardless of their original path categories.}
\label{tab:fix_perf}
\begin{tabular*}{\textwidth}{@{\extracolsep{\fill}}cl|c|c|c|c@{}}
\toprule
Dataset & Method &
\makecell{Original Full Dice} &
\makecell{Modified Full Dice} &
\makecell{Random-$Y_t$ Dice} &
\makecell{Shuffle-$Y_t$ Dice} \\
\midrule
\multirow[c]{5}{*}{\textbf{BTCV}} & \methodUniSegDiff & 85.82$\pm$1.86 & 84.12$\pm$1.31 & 15.06$\pm$1.33 & 19.08$\pm$1.44 \\
& \methodSSB & 84.01$\pm$1.70 & 82.10$\pm$2.00 & 8.81$\pm$2.12 & 13.69$\pm$1.46 \\
& \methodLEAF & 85.65$\pm$1.62 & 84.91$\pm$1.35 & 2.61$\pm$1.59 & 4.53$\pm$1.86 \\
& \methodColdSegDiff & 79.40$\pm$2.11 & 74.73$\pm$1.37 & 6.08$\pm$1.30 & 7.52$\pm$1.27 \\
& \methodCDAL & 81.00$\pm$1.54 & 73.42$\pm$1.35 & 13.95$\pm$1.35 & 18.56$\pm$1.39 \\
\cmidrule(lr){1-6}
\multirow[c]{5}{*}{\textbf{ACDC}} & \methodUniSegDiff & 82.76$\pm$1.72 & 81.19$\pm$1.48 & 7.95$\pm$1.33 & 11.28$\pm$1.28 \\
& \methodSSB & 81.29$\pm$1.77 & 78.49$\pm$2.34 & 18.27$\pm$1.29 & 21.20$\pm$1.09 \\
& \methodLEAF & 85.55$\pm$1.59 & 80.38$\pm$1.34 & 2.56$\pm$1.49 & 4.70$\pm$1.72 \\
& \methodColdSegDiff & 83.51$\pm$2.13 & 73.32$\pm$1.39 & 5.41$\pm$1.43 & 4.23$\pm$1.50 \\
& \methodCDAL & 87.63$\pm$1.76 & 82.18$\pm$1.33 & 14.42$\pm$1.44 & 17.52$\pm$1.30 \\
\cmidrule(lr){1-6}
\multirow[c]{5}{*}{\textbf{ISIC2018}} & \methodUniSegDiff & 89.27$\pm$1.72 & 82.13$\pm$1.76 & 17.29$\pm$1.40 & 17.74$\pm$1.37 \\
& \methodSSB & 87.07$\pm$2.04 & 84.19$\pm$2.32 & 13.98$\pm$1.85 & 15.79$\pm$2.10 \\
& \methodLEAF & 88.44$\pm$1.58 & 90.17$\pm$1.34 & 2.72$\pm$1.66 & 60.68$\pm$1.23 \\
& \methodColdSegDiff & 87.12$\pm$1.64 & 76.08$\pm$1.29 & 9.95$\pm$1.35 & 10.53$\pm$1.28 \\
& \methodCDAL & 88.47$\pm$1.76 & 83.73$\pm$1.32 & 20.04$\pm$1.36 & 26.49$\pm$1.22 \\
\bottomrule
\end{tabular*}
\end{table*}

\begin{table*}[!t]
\centering
\small
\caption{Noise-first intervention decisions at \(\kappa=1.0\) from ten seed-matched pairs. \(\Delta\)Dice is audit condition minus Modified Full, and \(\delta_r\) is the seed-level SD of Modified Full. Mean paired \(\Delta\)Dice is reported as \(\bar{\Delta}\pm t_{0.95,9}\mathrm{SE}_{\Delta}\), corresponding to the 90\% CI \([\mathrm{LCB}_{95},\mathrm{UCB}_{95}]\). SR denotes supported state reliance. All entries correspond to rerouted HC-Hybrid variants.}
\label{tab:fix_stat}
\begin{tabular*}{\textwidth}{@{\extracolsep{\fill}}l|c|cc|cc@{}}
\toprule
\multirow{2}{*}{Method} & \multirow{2}{*}{\(\delta_r\) (pp)} &
\multicolumn{2}{c|}{Random-\(Y_t\)} &
\multicolumn{2}{c}{Shuffle-\(Y_t\)} \\
& & Mean \(\Delta\)Dice (pp) & Decision & Mean \(\Delta\)Dice (pp) & Decision \\
\midrule
\multicolumn{6}{@{}l}{\textbf{BTCV}} \\[-0.3ex]
\methodUniSegDiff & 1.31 & $-69.06\pm0.13$ & SR & $-65.04\pm0.20$ & SR \\
\methodSSB & 2.00 & $-73.28\pm1.53$ & SR & $-68.40\pm1.20$ & SR \\
\methodLEAF & 1.35 & $-82.30\pm0.39$ & SR & $-80.38\pm0.55$ & SR \\
\methodColdSegDiff & 1.37 & $-68.64\pm0.15$ & SR & $-67.20\pm0.16$ & SR \\
\methodCDAL & 1.35 & $-59.46\pm0.10$ & SR & $-54.85\pm0.29$ & SR \\
\cmidrule(lr){1-6}
\multicolumn{6}{@{}l}{\textbf{ACDC}} \\[-0.3ex]
\methodUniSegDiff & 1.48 & $-73.24\pm0.94$ & SR & $-69.91\pm0.98$ & SR \\
\methodSSB & 2.34 & $-60.22\pm1.24$ & SR & $-57.29\pm1.38$ & SR \\
\methodLEAF & 1.34 & $-77.83\pm0.40$ & SR & $-75.68\pm0.73$ & SR \\
\methodColdSegDiff & 1.39 & $-67.91\pm0.16$ & SR & $-69.09\pm0.35$ & SR \\
\methodCDAL & 1.33 & $-67.76\pm0.15$ & SR & $-64.66\pm0.42$ & SR \\
\cmidrule(lr){1-6}
\multicolumn{6}{@{}l}{\textbf{ISIC2018}} \\[-0.3ex]
\methodUniSegDiff & 1.76 & $-64.83\pm0.82$ & SR & $-64.38\pm0.81$ & SR \\
\methodSSB & 2.32 & $-70.21\pm1.77$ & SR & $-68.40\pm1.52$ & SR \\
\methodLEAF & 1.34 & $-87.45\pm0.47$ & SR & $-29.49\pm0.17$ & SR \\
\methodColdSegDiff & 1.29 & $-66.14\pm0.11$ & SR & $-65.55\pm0.09$ & SR \\
\methodCDAL & 1.32 & $-63.69\pm0.13$ & SR & $-57.24\pm0.18$ & SR \\
\bottomrule
\end{tabular*}
\end{table*}

All 30 modified-method comparisons met the SR criterion, whereas the corresponding 30 original-method comparisons in \tabref{tab:stat_all_datasets} were Preserved. Even the smallest modified loss occurred for \methodLEAF{}/ISIC2018 under \shuffleyt{} and was \(-29.49\) pp (90\% CI, \([-29.66,-29.31]\)), compared with a Modified Full margin of \(1.34\) pp. The Preserved-to-SR transition therefore replicated across five methods, three datasets, and both audited state properties. The same magnitude pattern persisted after rerouting, with \shuffleyt{} producing smaller absolute Dice losses than \randyt{} in 14 of 15 settings.

The intervention scope differed across methods. For \methodUniSegDiff{} and \methodSSB{}, the modification relocated existing segmentation supervision from the LC-Hybrid bypass to the reconstruction, providing the tighter within-method test of path coupling. For LEAF, \methodColdSegDiff{}, and \methodCDAL{}, the redesign additionally changed the Directly Predicted Quantity from Mask to Noise and therefore constituted a broader noise-first modification. All five variants nevertheless showed the same categorical transition, with the first two providing the more specific attribution to bypass removal.

Modified Full Dice was lower than Original Full Dice in 14 of the 15 method-dataset settings, with LEAF/ISIC2018 as the sole exception (\tabref{tab:fix_perf}). These Original Full-Modified Full differences are descriptive under the primary analysis. Thus, descriptively, the rerouting produced SR under both audits without improving Full Dice in 14 of 15 settings. 

\subsection{Deterministic Utility Is Distinct from State Reliance}

Deterministic utility did not follow the state-reliance pattern. Across all 35 method-dataset settings, \corenodiff{} was Similar or Better than Full in 28. Among the 20 settings in which both state audits met the SR criterion, \corenodiff{} was Similar or Better in 16, with all nine Better decisions occurred in this state-reliant subset (\tabref{tab:stat_all_datasets}). This pattern shows that native state reliance and deterministic endpoint recoverability are empirically distinct.

All twelve HC-Hybrid \corenodiff{} comparisons were Similar or Better, comprising four Similar and eight Better decisions. Diff-Coupled results were more heterogeneous, with three Similar, one Better, three Worse, and one Inconclusive. Among the 15 bypass-capable settings, \corenodiff{} was Similar in 12, Worse in one, and Inconclusive in two.

Endpoint recoverability also varied by method and dataset. For \methodMedSegDiffVOne{}, for example, \corenodiff{} was Better on BTCV (\(+4.61\) pp) but Worse on ACDC and ISIC2018 (\(-9.72\) and \(-7.57\) pp, respectively).

\noindent\textbf{Margin sensitivity.} Re-evaluating the decisions over \(\kappa\in\{0.5,1.0,1.5,2.0\}\) left all 70 original-method state-audit decisions and all 30 modified-method decisions unchanged. The complete Path Category separation and every intervention-induced Preserved-to-SR transition therefore persisted across the tested fourfold range of margins. Only nine of the 35 \corenodiff{} decisions were tolerance-dependent over the tested range, and all moved toward Similar as the margin widened, with all nine were Similar at \(\kappa=2.0\) (Appendix~\ref{app:margin_sensitivity}, \tabref{tab:margin_sensitivity}).

\section{Discussion and Conclusion}

Taken together, the results separate two attribution questions in fully supervised diffusion segmentation: whether the native formulation relies on the noisy state, and whether diffusion-specific computation provides a deterministic performance advantage over image-only prediction. State reliance followed the supervision path to the evaluated mask. Methods whose supervised routes could bypass noised-quantity reconstruction recovered reference performance under state disruption, whereas coupled routes did not, and rerouting originally bypass-capable methods onto a noise-first route induced the corresponding reliance. The HC-Hybrid noise-loss-removal ablation further showed that this response persisted without an explicit noise-loss term in the audited variants. Deterministic utility followed a different pattern, with image-only counterparts often recovering comparable or better endpoint performance even when the native formulation was state-reliant. The rerouted references also generally did not improve deterministic performance. One plausible explanation is that the original architectures and optimization settings were developed for their native formulations; forcing segmentation supervision through noised-quantity prediction and reconstruction may therefore introduce a mismatch between the model design and the modified objective, particularly in the broader redesigns that also change the directly predicted quantity. State reliance can thus be induced by formulation, but its presence does not by itself establish deterministic benefit from diffusion-specific computation.

These findings point to a broader attribution problem in conditional prediction. A conventional predictor learns a condition-to-target relation \(C\!\to\!Y\), whereas a conditional diffusion formulation additionally supplies a target-derived noisy state, yielding \((C,Y_t,t)\!\to\!Y\). When the condition already supports a strong predictor \(h(C)\), endpoint performance alone establishes neither that the added state is required nor that diffusion-specific computation provides predictive value beyond condition-only prediction. Segmentation instantiates this structure as \(\I\!\to\!\Yzero\), and analogous questions may arise in tasks such as super-resolution, where the conditioning input already supports direct prediction of the target before a noisy target state is introduced. Crucially, condition-only endpoint recoverability does not imply a state-free native computation path: as shown in Section~\ref{sec:formal_rationale}, a reconstruction route may still require a supplied-state-dependent intermediate prediction, even when that dependence cancels in the composed endpoint. The relevant mechanistic object is therefore the complete supervision path to the evaluated output, rather than the nominal presence of diffusion, the directly predicted quantity, or endpoint performance alone.

This perspective refines prior state-construction and image-only-control studies~\cite{fu2023aligning,fu2023recycling,ottl2024diffseg} by treating state reliance and deterministic predictive utility as separate attribution targets. Claims of diffusion-specific contribution should therefore pair mechanism-specific reliance evidence with a matched condition-only counterfactual testing whether the claimed capability remains recoverable without diffusion-specific computation. Our conclusions are limited to fully supervised segmentation under single-reference deterministic endpoint evaluation and do not assess potential benefits for distributional prediction or modeling multiple plausible masks.

\section*{Resource Availability}

\noindent\textbf{Lead contact.}
Further information and requests for resources should be directed to and will
be fulfilled by the lead contact, Hengzhuo Yang
(\href{mailto:yang.hengz@northeastern.edu}{yang.hengz@northeastern.edu}).

\noindent\textbf{Materials availability.}
This study did not generate new physical materials.

\noindent\textbf{Data and code availability.}
The BTCV/Synapse dataset~\cite{landman2015btcv} is available at
\url{https://www.synapse.org/Synapse:syn3193805/wiki/217789}.
The ACDC dataset~\cite{bernard2018acdc} is available at
\url{https://www.creatis.insa-lyon.fr/Challenge/acdc/databases.html}.
The ISIC2018 dataset~\cite{codella2019isic2018} is available at
\url{https://challenge.isic-archive.com/data/#2018}.

Code is available at \url{https://github.com/Hengzhuo-Yang/Rethinking-Diffusion-Segmentation}.

\section*{Acknowledgments}
This research received no specific grant from funding agencies in the public, commercial, or not-for-profit sectors.

\section*{Author Contributions}
\noindent
\textbf{Hengzhuo Yang:} Writing - original draft, Writing - review \& editing, Visualization, Software, Methodology, Conceptualization, Formal analysis.
\textbf{Yuming Zeng:} Writing - review \& editing, Conceptualization, Validation.
\textbf{Yuling Yang:} Writing - review \& editing, Conceptualization, Validation.

\section*{Declaration of Interests}
The authors declare no competing interests.

\section*{Declaration of generative AI and AI-assisted technologies in the writing process}
During the preparation of this work, the authors used ChatGPT (OpenAI) 
to improve the language and readability of the manuscript. 
After using this tool, the authors reviewed and edited the content as needed 
and take full responsibility for the content of the published article.

\bibliographystyle{IEEEtran}
\bibliography{refs}

\appendices

\section{Detailed Supervision-Path Structure}
\label{app:taxonomy_details}

This appendix expands the operational assignment procedure in \figref{fig:path_category_assignment}. Path Category is determined by the documented loss-supported routes to the evaluated mask. \figref{fig:prediction_supervision_routes} shows the detailed prediction and supervision structures represented by the four categories, including the locations of segmentation-related supervision relative to noised-quantity reconstruction. 

\begin{figure*}[!t]
\centering
\safeincludegraphics[width=0.7\textwidth]{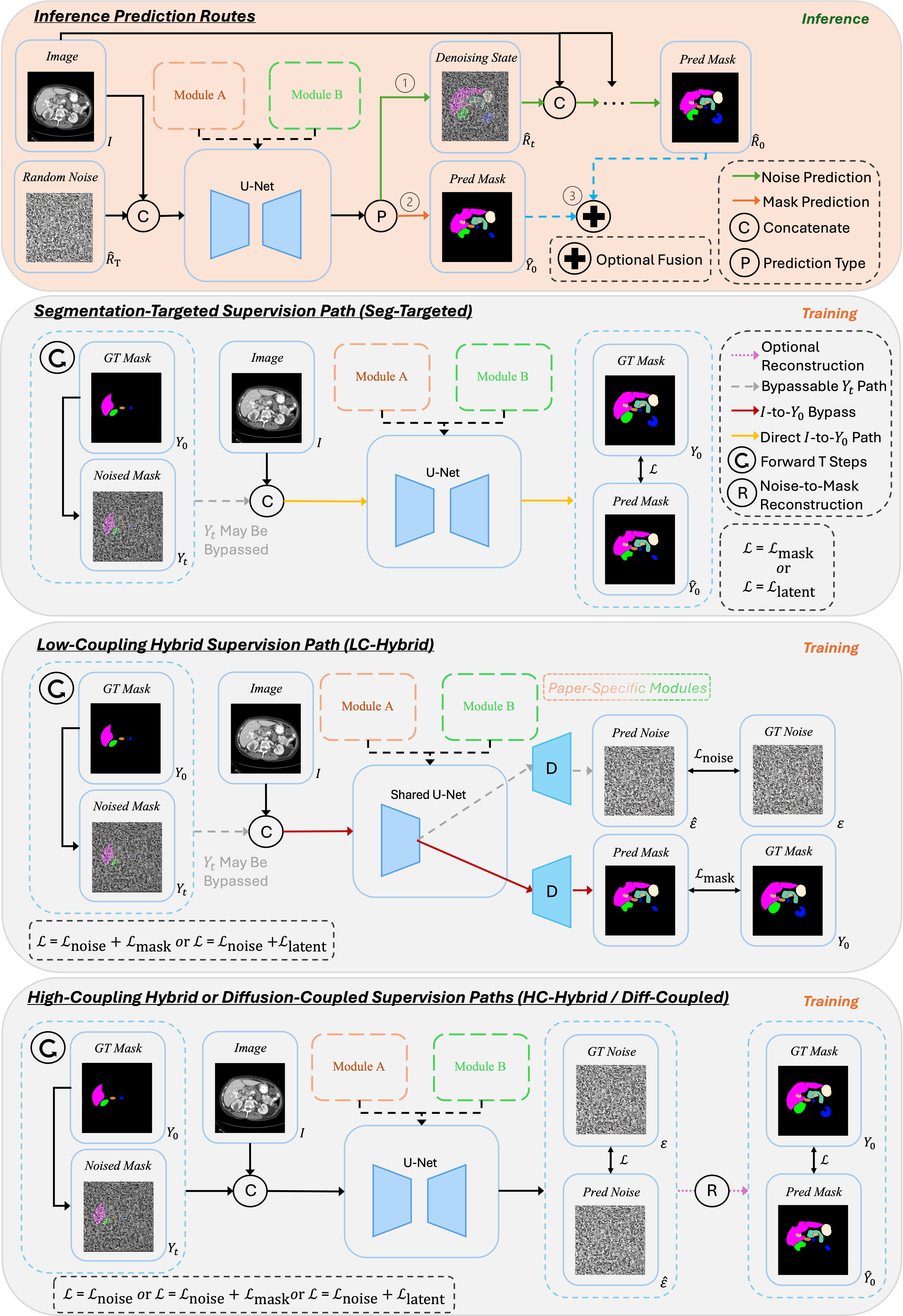}
\caption{Detailed prediction and supervision routes underlying the four Path Categories. The inference panel traces the predicted quantities used to obtain the evaluated mask, whereas the training panels locate segmentation-related supervision relative to noised-quantity reconstruction. Either route may be embedded in an iterative inference method.}
\label{fig:prediction_supervision_routes}
\end{figure*}

\section{Experimental Setup and Data Splits}
\label{app:experimental_setup}

Every run used one NVIDIA GeForce RTX 5090. We followed each method's released
implementation where available and its reported training and inference
configuration otherwise. Memory-driven reductions in per-device batch size
were held fixed across Full and its matched conditions. No audit-specific
tuning was performed.

All conditions within an implementation used the same fixed partition.
BTCV/Synapse~\cite{landman2015btcv} followed the TransUNet 18/12-volume
split~\cite{chen2024transunet}. We fixed two held-out volumes (295 slices) for
checkpoint selection and used the remaining 10 volumes (1,273 slices) for
evaluation, yielding an 18/2/10 case-level partition with 2,211 training
slices. ACDC~\cite{bernard2018acdc} used a 70/10/20-subject split with
1,304/182/416 training, validation, and test slices. ISIC2018~\cite{codella2019isic2018}
used the official 2,594/100/1,000-image split. The released binary-mask
formulation of \methodBerDiff{} was incompatible with the native four-class
ACDC task, so that setting was unavailable.

\section{Formal Analysis}
\label{app:path_hypothesis_proofs}

This appendix supplies the assumptions and proofs for
Section~\ref{sec:formal_rationale}. The audit conditions are defined in
Section~\ref{sec:audit_protocol}. The argument proceeds in four stages. The
first formalizes the image-only solution used by a conventional
non-diffusion segmentor and its invariance when represented within a
diffusion-style segmentation-output predictor class. The second compares the
information available for predicting the independently sampled noise target.
The third characterizes the supplied-state dependence required for a Gaussian \(\epsilon\)-prediction reconstruction to realize an image-only endpoint. The fourth determines when retained-noise correctness and clean-target reconstruction are compatible. Each result has a different role in the route-level hypotheses. Together, they provide population or pointwise statements that motivate the route-level predictions, while empirical Dice outcomes and audit decisions are determined separately by finite retraining.

Throughout this appendix, \(\Yzero\) denotes the clean target associated with
\(\I\) for a sampled case and is not a reverse-time state. The timestep \(t\)
indexes the diffusion coefficients, the supplied state \(\Yt^{(c)}\), and the
corresponding timestep-conditioned predictor call. The population Bayes-risk
statements for noise recovery average over the joint law including \(t\),
whereas the reconstruction identities hold pointwise for each realized \(t\)
with \(a_t,b_t>0\). The analysis concerns the training-time state constructions
defined in Section~\ref{sec:audit_protocol} and does not model the sequence of
model-generated states in a multistep reverse trajectory.

\subsection{Non-Diffusion Segmentation Reference and Audit Invariance}

Let \(\tau_m(\Yzero)\) denote the segmentation target used in this analysis,
either the clean mask in pixel space or a deterministic encoding of that mask
in a latent segmentation space. Let \(\ell_m\) be a nonnegative loss defined
on the corresponding segmentation-output space, and let \(\mathcal H_I\)
denote a class of measurable image-only predictors taking values in that
space. For \(h\in\mathcal H_I\), define
\begin{equation}
\begin{aligned}
    \mathcal R_{\mathrm{img}}(h)
    &:=
    \mathbb E\!\left[
      \ell_m\!\left(h(\I),\tau_m(\Yzero)\right)
    \right],\\
    \mathcal R_{\mathrm{img}}^\star
    &:=
    \inf_{h\in\mathcal H_I}
    \mathcal R_{\mathrm{img}}(h).
\end{aligned}
    \label{eq:app_image_only_segmentation_risk}
\end{equation}
This is the population segmentation risk minimized by a conventional
non-diffusion segmentor over image-only predictors. The next proposition gives
the exact image-only solution for this reference under the realizable
abstraction used later for Gaussian noise recovery.

\begin{proposition}[Image-only segmentation optimum]
\label{prop:app_image_only_optimum}
Assume that \(\Yzero=\phi(\I)\) almost surely for a measurable map \(\phi\),
that \(h_0=\tau_m\circ\phi\) belongs to \(\mathcal H_I\), and that
\(\ell_m(z,z)=0\) for every segmentation target \(z\). Then
\begin{equation}
    \mathcal R_{\mathrm{img}}(h_0)
    =
    \mathcal R_{\mathrm{img}}^\star
    =0.
    \label{eq:app_image_only_optimum}
\end{equation}
\end{proposition}

\begin{proof}
The realizable assumption gives
\[
    h_0(\I)
    =
    \tau_m\!\left(\phi(\I)\right)
    =
    \tau_m(\Yzero)
    \quad\text{almost surely}.
\]
Substituting this equality into
Eq.~\eqref{eq:app_image_only_segmentation_risk} yields
\[
    \mathcal R_{\mathrm{img}}(h_0)
    =
    \mathbb E\!\left[
      \ell_m\!\left(
        \tau_m(\Yzero),
        \tau_m(\Yzero)
      \right)
    \right]
    =0.
\]
Because \(\ell_m\) is nonnegative, no predictor can have risk below zero.
Therefore \(h_0\) attains the global infimum and
\(\mathcal R_{\mathrm{img}}^\star=0\).
\end{proof}

The zero-risk conclusion requires realizability and an exact-target loss. The
audit-invariance result below is broader and holds for any fixed image-only
candidate whose risk is finite.

For a complete segmentation-output predictor \(g\), define its segmentation
risk under condition \(c\in\{\mathrm{Full},\mathrm{Rand},\mathrm{Shuf}\}\) by
\begin{equation}
    \mathcal R_m^{(c)}(g)
    =
    \mathbb E_c\!\left[
      \ell_m\!\left(
        g(\I,\Yt^{(c)},t),
        \tau_m(\Yzero)
      \right)
    \right].
    \label{eq:app_segmentation_risk}
\end{equation}
Here \(\mathbb E_c\) denotes expectation under the condition-specific supplied
state and training tuple. The audits change that state while preserving the
joint law of \((\I,\Yzero)\). Equation~\eqref{eq:app_segmentation_risk} is an
analytical segmentation-risk quantity and need not itself appear as a term in
every method's training objective.

\begin{lemma}[Image-only segmentation-risk invariance]
\label{lem:app_image_only_invariance}
Assume that the complete segmentation-output predictor class in every condition contains
a common realization \(g_h\) satisfying
\[
    g_h(\I,\Yt^{(c)},t)=h(\I)
\]
for a fixed \(h\in\mathcal H_I\) and every condition \(c\). Then
\begin{equation}
\begin{aligned}
    \mathcal R_m^{(\mathrm{Full})}(g_h)
    &={}
    \mathcal R_m^{(\mathrm{Rand})}(g_h)
    =
    \mathcal R_m^{(\mathrm{Shuf})}(g_h)\\
    &={}
    \mathbb E\!\left[
      \ell_m\!\left(h(\I),\tau_m(\Yzero)\right)
    \right].
\end{aligned}
    \label{eq:app_image_only_invariance}
\end{equation}
\end{lemma}

\begin{proof}
Fix a condition \(c\). Substituting the assumed form of \(g_h\) into
Eq.~\eqref{eq:app_segmentation_risk} gives
\begin{align*}
    \mathcal R_m^{(c)}(g_h)
    &=
    \mathbb E_c\!\left[
      \ell_m\!\left(
        h(\I),
        \tau_m(\Yzero)
      \right)
    \right].
\end{align*}
The integrand depends only on \((\I,\Yzero)\). Its distribution is the same
under Full, Random, and Shuffle, so
\[
    \mathbb E_c\!\left[
      \ell_m\!\left(h(\I),\tau_m(\Yzero)\right)
    \right]
    =
    \mathbb E\!\left[
      \ell_m\!\left(h(\I),\tau_m(\Yzero)\right)
    \right]
\]
for every \(c\), which proves the equality.
\end{proof}

Taking \(h=h_0\) from Proposition~\ref{prop:app_image_only_optimum} makes the
common risk in Eq.~\eqref{eq:app_image_only_invariance} equal to zero whenever
all three complete segmentation-output predictor classes contain the corresponding
realization. For any other \(h\), the lemma still shows equality of its
segmentation risk across the audits.

The lemma is category-neutral. It establishes availability of a common
image-only solution within the diffusion-style segmentation-output predictor class and says
nothing about whether the documented objective supports it or finite
optimization selects it. Path Category provides that additional
supervision-path information. In \pathSegTargeted{} and in the
segmentation-supervised bypass of LC-Hybrid, segmentation supervision can
support this solution without first requiring an accurate noised-quantity
prediction.
Diff-Coupled has no corresponding segmentation loss, while HC-Hybrid has no
segmentation-supervised route that reaches the evaluated mask without first
requiring an accurate noised-quantity prediction. Either category may still represent the
same solution through an intermediate noised-quantity output. The
representation and its compatibility with noised-quantity supervision are
analyzed after the noise-risk result.

\subsection{Retained-Noise Information under Gaussian Prediction}

The remaining results specialize to Gaussian \(\epsilon\)-prediction with
squared-error loss. Treat the mask, state, and noise tensors as vectors in
\(\mathbb R^d\). Let \(\mathbf I_{\mathrm{id}}\) denote the corresponding
\(d\times d\) identity matrix. The risks below use unnormalized
\(\ell_2^2\); coordinate-averaged squared error divides every displayed risk
by \(d\).

All random variables are defined on a common probability space. Assume
\(\mathbb E\|\Yzero\|^2<\infty\),
\(\eps\sim\mathcal N(\mathbf 0,\mathbf I_{\mathrm{id}})\) independently of
\((\I,\Yzero,t)\), and a finite diffusion schedule. The coefficients
\(a_t,b_t>0\) are deterministic functions of the observed \(t\) and are
therefore bounded over this schedule. Exact recovery under Full additionally
uses the realizable assumption \(\Yzero=\phi(\I)\) almost surely for a
measurable \(\phi\).

For Random-\(Y_t\), let
\(\etaNoise\sim\mathcal N(\mathbf 0,\mathbf I_{\mathrm{id}})\) be independent
of \((\I,\Yzero,t,\eps)\). For Shuffle-\(Y_t\), let \(\Yzero'\) have the same
marginal distribution as \(\Yzero\) and be independent of
\((\I,\Yzero,t,\eps,\etaNoise)\). Shuffle retains the current \(\eps\) in the
supplied state and also retains it as the supervision target. The three states
are
\begin{equation}
\begin{aligned}
    \Yt^{(\mathrm{Full})}
    &=a_t\Yzero+b_t\eps,\\
    \Yt^{(\mathrm{Rand})}
    &=\etaNoise,\\
    \Yt^{(\mathrm{Shuf})}
    &=a_t\Yzero'+b_t\eps.
\end{aligned}
    \label{eq:app_audit_states}
\end{equation}

Before comparing these conditions, the next lemma states the common
squared-error calculation used for every risk. Let \(X\) be any random input,
and let the infimum below range over measurable predictors \(f\) satisfying
\(\mathbb E\|f(X)\|^2<\infty\):
\begin{equation}
    \mathcal R_{\eps}(X)
    :=
    \inf_f
    \mathbb E\|f(X)-\eps\|^2.
    \label{eq:app_general_noise_risk}
\end{equation}

\begin{lemma}[Squared-error projection]
\label{lem:app_squared_error_projection}
Let \(m_X=\mathbb E[\eps\mid X]\). Then \(m_X\) attains the infimum in
Eq.~\eqref{eq:app_general_noise_risk}, and
\begin{equation}
\begin{aligned}
    \mathbb E\|f(X)-\eps\|^2
    ={}&
    \mathbb E\|f(X)-m_X\|^2
    +
    \mathbb E\|m_X-\eps\|^2,\\
    \mathcal R_{\eps}(X)
    ={}&
    \mathbb E\!\left[
      \operatorname{tr}\operatorname{Var}(\eps\mid X)
    \right],\\
    d
    ={}&
    \mathcal R_{\eps}(X)+\mathbb E\|m_X\|^2.
\end{aligned}
    \label{eq:app_projection}
\end{equation}
\end{lemma}

\begin{proof}
For any admissible predictor \(f\), write
\[
    f(X)-\eps
    =
    \bigl(f(X)-m_X\bigr)+\bigl(m_X-\eps\bigr).
\]
The first term is a measurable function of \(X\). The second has conditional
mean zero because
\[
    \mathbb E[m_X-\eps\mid X]
    =m_X-\mathbb E[\eps\mid X]
    =\mathbf 0.
\]
The cross term in the squared expansion therefore satisfies
\begin{align*}
&\mathbb E\!\left[
  \bigl(f(X)-m_X\bigr)^\top
  \bigl(m_X-\eps\bigr)
\right]\\
&\quad=
\mathbb E\!\left[
  \bigl(f(X)-m_X\bigr)^\top
  \mathbb E[m_X-\eps\mid X]
\right]
=0.
\end{align*}
Expanding the squared norm now gives the first identity in
Eq.~\eqref{eq:app_projection}. Its first term is nonnegative and vanishes for
\(f(X)=m_X\) almost surely. Conditional expectation preserves square
integrability, so \(m_X\) is admissible and attains the infimum.

For a vector target,
\[
    \mathbb E\!\left[
      \|\eps-m_X\|^2\mid X
    \right]
    =
    \operatorname{tr}\operatorname{Var}(\eps\mid X).
\]
Taking expectations gives the second identity. Finally, substitute the zero
predictor \(f(X)=\mathbf 0\) into the first identity and use
\(\mathbb E\|\eps\|^2=\operatorname{tr}(\mathbf I_{\mathrm{id}})=d\). This
gives the last identity, which separates total noise energy into the part
unpredictable from \(X\) and the part captured by the conditional mean.
\end{proof}

The first information comparison uses only the image and timestep versus the
matched Full inputs.

\begin{proposition}[Image-timestep and Full noise risks]
\label{prop:app_full_information}
Let \(X_I=(\I,t)\) and
\(X_{\mathrm{Full}}=(\I,\Yt^{(\mathrm{Full})},t)\). Under the assumptions
above, including \(\Yzero=\phi(\I)\) almost surely,
\begin{equation}
    \mathcal R_{\eps}(X_I)=d,
    \qquad
    \mathcal R_{\eps}(X_{\mathrm{Full}})=0.
    \label{eq:app_full_information}
\end{equation}
\end{proposition}

\begin{proof}
Independence of \(\eps\) from \((\I,t)\), together with
\(\mathbb E[\eps]=\mathbf 0\), gives
\[
    \mathbb E[\eps\mid X_I]=\mathbf 0.
\]
Lemma~\ref{lem:app_squared_error_projection} therefore yields
\[
    \mathcal R_{\eps}(X_I)
    =
    \mathbb E\|\eps\|^2
    =d.
\]

For the matched Full inputs, realizability gives
\begin{equation}
    \eps
    =
    \frac{\Yt^{(\mathrm{Full})}-a_t\Yzero}{b_t}
    =
    \frac{\Yt^{(\mathrm{Full})}-a_t\phi(\I)}{b_t}.
    \label{eq:app_full_noise_recovery}
\end{equation}
The final expression is a measurable function of
\((\I,\Yt^{(\mathrm{Full})},t)\). Hence \(\eps\) is determined by
\(X_{\mathrm{Full}}\), its conditional variance given these inputs is zero,
and Lemma~\ref{lem:app_squared_error_projection} gives
\(\mathcal R_{\eps}(X_{\mathrm{Full}})=0\).
\end{proof}

Proposition~\ref{prop:app_full_information} separates the segmentation target from the
realized noise target. The image may determine \(\Yzero\) under the realizable
abstraction while remaining uninformative about the independently sampled
\(\eps\). The matched Full state supplies the additional relation needed to
recover that realization exactly.

The next proposition compares the complete inputs supplied under the three
audit conditions.

\begin{proposition}[Audit-state noise risks]
\label{prop:app_gaussian_risks}
For \(c\in\{\mathrm{Full},\mathrm{Rand},\mathrm{Shuf}\}\), let
\(X_c=(\I,\Yt^{(c)},t)\) and
\(\mathcal R_c=\mathcal R_{\eps}(X_c)\). Then
\begin{equation}
\begin{aligned}
    \mathcal R_{\mathrm{Full}}=0
    &\le
    \mathcal R_{\mathrm{Shuf}}
    <
    \mathcal R_{\mathrm{Rand}}=d,\\
    \mathcal R_{\mathrm{Shuf}}
    &=
    \mathbb E\!\left[
      \frac{a_t^2}{b_t^2}
      \operatorname{tr}
      \operatorname{Var}
      (\Yzero'\mid X_{\mathrm{Shuf}})
    \right].
\end{aligned}
    \label{eq:app_audit_risks}
\end{equation}
If the displayed conditional-covariance trace is positive on an event of
positive probability, then \(\mathcal R_{\mathrm{Shuf}}>0\). This condition
means that the donor mask is not always determined by the Shuffle inputs.
\end{proposition}

\begin{proof}
The Full equality is the second result in
Proposition~\ref{prop:app_full_information}.

For Random, \(X_{\mathrm{Rand}}=(\I,\etaNoise,t)\). The retained target
\(\eps\) is independent of all three components, so
\(\mathbb E[\eps\mid X_{\mathrm{Rand}}]=\mathbf 0\). The projection lemma
then gives
\[
    \mathcal R_{\mathrm{Rand}}
    =
    \mathbb E\|\eps\|^2
    =d.
\]
Thus the independent Gaussian state supplies no information about the retained
noise beyond the image and timestep.

For Shuffle, rearranging the third state in
Eq.~\eqref{eq:app_audit_states} gives
\[
    \eps
    =
    \frac{\Yt^{(\mathrm{Shuf})}}{b_t}
    -
    \frac{a_t}{b_t}\Yzero'.
\]
Conditioning both sides on \(X_{\mathrm{Shuf}}\) yields the optimal predictor
\begin{equation}
    m_{\mathrm{Shuf}}
    :=
    \mathbb E[\eps\mid X_{\mathrm{Shuf}}]
    =
    \frac{\Yt^{(\mathrm{Shuf})}}{b_t}
    -
    \frac{a_t}{b_t}
    \mathbb E[\Yzero'\mid X_{\mathrm{Shuf}}].
    \label{eq:app_shuffle_conditional_mean}
\end{equation}
Because \(t\) is a component of \(X_{\mathrm{Shuf}}\), the ratio
\(a_t/b_t\) is fixed under this conditioning.
Subtracting the original expression for \(\eps\) gives its prediction
residual:
\[
    m_{\mathrm{Shuf}}-\eps
    =
    \frac{a_t}{b_t}
    \left(
      \Yzero'-\mathbb E[\Yzero'\mid X_{\mathrm{Shuf}}]
    \right).
\]
Squaring, taking expectations, and using the conditional-variance identity in
Lemma~\ref{lem:app_squared_error_projection} gives the Shuffle expression in
Eq.~\eqref{eq:app_audit_risks}. It is nonnegative and is strictly positive
under the stated donor-uncertainty condition.

It remains to prove the strict upper bound
\(\mathcal R_{\mathrm{Shuf}}<d\). The last identity in
Eq.~\eqref{eq:app_projection} gives
\[
    d
    =
    \mathcal R_{\mathrm{Shuf}}
    +
    \mathbb E\|m_{\mathrm{Shuf}}\|^2.
\]
Equality \(\mathcal R_{\mathrm{Shuf}}=d\) would therefore require
\(m_{\mathrm{Shuf}}=\mathbf 0\) almost surely. Because
\(\Yt^{(\mathrm{Shuf})}\) and \(t\) are components of
\(X_{\mathrm{Shuf}}\), the law of iterated expectations would then give
\begin{align*}
    \mathbb E\!\left[
      \eps\bigl(\Yt^{(\mathrm{Shuf})}\bigr)^\top\mid t
    \right]
    &=
    \mathbb E\!\left[
      \mathbb E[\eps\mid X_{\mathrm{Shuf}}]
      \bigl(\Yt^{(\mathrm{Shuf})}\bigr)^\top\mid t
    \right]\\
    &=\mathbf 0.
\end{align*}
Direct calculation from Eq.~\eqref{eq:app_audit_states} instead gives
\begin{align*}
    \mathbb E\!\left[
      \eps\bigl(\Yt^{(\mathrm{Shuf})}\bigr)^\top\mid t
    \right]
    &=
    a_t\mathbb E[\eps{\Yzero'}^\top\mid t]
    +
    b_t\mathbb E[\eps\eps^\top\mid t]\\
    &=b_t\mathbf I_{\mathrm{id}}\ne\mathbf 0.
\end{align*}
The donor-noise term is zero by independence, and the second moment of
\(\eps\) is \(\mathbf I_{\mathrm{id}}\). The result is nonzero because
\(b_t>0\), contradicting the assumed equality. Hence
\(\mathcal R_{\mathrm{Shuf}}<d=\mathcal R_{\mathrm{Rand}}\), which completes
the ordering.
\end{proof}

Proposition~\ref{prop:app_gaussian_risks} compares retained-noise
recoverability. Full supplies enough information for exact recovery, Random
supplies none, and Shuffle supplies some because it retains the current
\(\eps\). These risks do not measure endpoint performance or order the two
audits by Dice loss or finite-training effect size.

The proposition uses an independently drawn donor. The experiment instead
reshuffles donors without replacement within each batch while excluding
self-matches. The proposition is therefore an independent-donor idealization
of the implemented Shuffle audit.

\subsection{Gaussian Endpoint Representation and Reconstruction Compatibility}

The Bayes risks above concern the intermediate output supervised against the
retained \(\eps\). A different question is whether the complete
noise-to-mask route can represent an image-to-mask solution. The next
proposition separates these two output levels.

\begin{lemma}[Supplied-state dependence of image-only endpoint realization]
\label{prop:app_image_only_reconstruction}

Fix a timestep \(t\) with \(a_t,b_t>0\), and let
\(h:\mathcal I\to\mathbb R^d\) be a measurable clean-mask predictor.
Suppose that, for the same \((\I,t)\), two supplied states
\(\mathbf y\) and \(\widetilde{\mathbf y}\) satisfy
\[
    \frac{\mathbf y-b_t\psi(\I,\mathbf y,t)}{a_t}
    =
    \frac{\widetilde{\mathbf y}
    -b_t\psi(\I,\widetilde{\mathbf y},t)}{a_t}
    =
    h(\I).
\]
Then
\begin{equation}
    \psi(\I,\mathbf y,t)
    -
    \psi(\I,\widetilde{\mathbf y},t)
    =
    \frac{\mathbf y-\widetilde{\mathbf y}}{b_t}.
    \label{eq:app_state_dependence}
\end{equation}
Hence, if \(\mathbf y\neq\widetilde{\mathbf y}\), the required intermediate
outputs are distinct, and no supplied-state-invariant intermediate of the form
\(v(\I,t)\) can realize the same endpoint \(h(\I)\) at both states.
\end{lemma}

\begin{proof}
Subtracting the two reconstruction equalities gives
\[
    \mathbf y-\widetilde{\mathbf y}
    =
    b_t\!\left[
      \psi(\I,\mathbf y,t)-\psi(\I,\widetilde{\mathbf y},t)
    \right].
\]
Division by \(b_t>0\) yields
Eq.~\eqref{eq:app_state_dependence}.
\end{proof}

The lemma characterizes the intermediate required for structural representability. We next ask whether a clean-target-correct intermediate can simultaneously match the retained noise target.

\begin{proposition}[Retained-noise and clean-target compatibility]
\label{prop:app_reconstruction_compatibility}

For any \(c\in\{\mathrm{Full},\mathrm{Rand},\mathrm{Shuf}\}\), define
\begin{equation}
    \widehat{\Yzero}^{\,\mathrm{rec},(c)}
    =
    \frac{\Yt^{(c)}-b_t\widehat{\eps}^{(c)}}{a_t}.
    \label{eq:app_reconstruction_setup}
\end{equation}
Relative to the Full construction
\(\Yt^{(\mathrm{Full})}=a_t\Yzero+b_t\eps\), the reconstruction error relative
to the clean target is
\begin{equation}
    \widehat{\Yzero}^{\,\mathrm{rec},(c)}-\Yzero
    =
    \frac{\Yt^{(c)}-\Yt^{(\mathrm{Full})}}{a_t}
    +
    \frac{b_t}{a_t}
    \left(\eps-\widehat{\eps}^{(c)}\right).
    \label{eq:app_reconstruction_error}
\end{equation}
Consequently,
\begin{equation}
\begin{aligned}
    \widehat{\eps}^{(c)}=\eps
    &\quad\Longrightarrow\quad
    \widehat{\Yzero}^{\,\mathrm{rec},(c)}-\Yzero
    =
    \frac{\Yt^{(c)}-\Yt^{(\mathrm{Full})}}{a_t},\\
    \widehat{\Yzero}^{\,\mathrm{rec},(c)}=\Yzero
    &\quad\Longrightarrow\quad
    \widehat{\eps}^{(c)}
    =
    \eps+
    \frac{\Yt^{(c)}-\Yt^{(\mathrm{Full})}}{b_t}.
\end{aligned}
    \label{eq:app_reconstruction_implications}
\end{equation}
Therefore, the retained-noise-correct choice
\(\widehat{\eps}^{(c)}=\eps\) also reconstructs the clean target exactly if and
only if
\(\Yt^{(c)}=\Yt^{(\mathrm{Full})}\).

\end{proposition}

\begin{proof}
Substituting
\[
    \Yzero
    =
    \frac{\Yt^{(\mathrm{Full})}-b_t\eps}{a_t}
\]
into Eq.~\eqref{eq:app_reconstruction_setup} gives
Eq.~\eqref{eq:app_reconstruction_error}. The two implications follow by
setting \(\widehat{\eps}^{(c)}=\eps\) and
\(\widehat{\Yzero}^{\,\mathrm{rec},(c)}=\Yzero\), respectively. Hence both
equalities hold simultaneously if and only if
\(\Yt^{(c)}=\Yt^{(\mathrm{Full})}\).
\end{proof}

The three conditions make these requirements explicit. Under Full,
\[
    \widehat{\Yzero}^{\,\mathrm{rec},(\mathrm{Full})}-\Yzero
    =
    \frac{b_t}{a_t}
    \left(\eps-\widehat{\eps}^{(\mathrm{Full})}\right).
\]
Thus, under the native construction, exact retained-noise prediction and exact
clean-target reconstruction are equivalent.

Under Random, the state mismatch is
\(\etaNoise-\Yt^{(\mathrm{Full})}\). Exact prediction of the retained noise
gives
\begin{equation}
    \widehat{\eps}^{(\mathrm{Rand})}=\eps
    \quad\Longrightarrow\quad
    \widehat{\Yzero}^{\,\mathrm{rec},(\mathrm{Rand})}
    =
    \frac{\etaNoise-b_t\eps}{a_t},
    \label{eq:app_random_exact_noise}
\end{equation}
whereas exact clean-target reconstruction requires
\begin{equation}
\begin{aligned}
    \widehat{\eps}^{(\mathrm{Rand})}
    &=
    \frac{\etaNoise-a_t\Yzero}{b_t}\\
    &=
    \eps+
    \frac{\etaNoise-\Yt^{(\mathrm{Full})}}{b_t}.
\end{aligned}
    \label{eq:app_random_clean_target}
\end{equation}
Conditional on \((\Yzero,t)\), the difference
\(\etaNoise-\Yt^{(\mathrm{Full})}\) has a nondegenerate Gaussian
distribution under the stated assumptions. Hence the Random state equals its
Full counterpart only on a probability-zero event, and the two exactness
requirements almost surely differ.

Under Shuffle, the state mismatch is
\(a_t(\Yzero'-\Yzero)\). Therefore
\begin{equation}
\begin{aligned}
    \widehat{\eps}^{(\mathrm{Shuf})}=\eps
    &\quad\Longrightarrow\quad
    \widehat{\Yzero}^{\,\mathrm{rec},(\mathrm{Shuf})}=\Yzero',\\
    \widehat{\Yzero}^{\,\mathrm{rec},(\mathrm{Shuf})}=\Yzero
    &\quad\Longrightarrow\quad
    \widehat{\eps}^{(\mathrm{Shuf})}
    =
    \eps+
    \frac{a_t}{b_t}(\Yzero'-\Yzero).
\end{aligned}
    \label{eq:app_shuffle_condition_specific}
\end{equation}
When \(\Yzero'\neq\Yzero\), exact retained-noise prediction and exact
clean-target reconstruction cannot both hold.

This identity is pointwise. In a Diff-Coupled route, a retained-noise-correct
output reconstructs the clean quantity implied by the audited state. For an HC-Hybrid route that combines reconstruction-based segmentation supervision with explicit retained-noise supervision, the retained-noise and clean-target requirements conflict whenever the state mismatch is nonzero; the tradeoff selected by a weighted finite-training objective remains method- and optimization-dependent.

\subsection{Formal Scope and Empirical Interpretation}
\label{app:theory_experiment_boundary}

The formal results motivate the route-level hypotheses by connecting
image-only representability, objective alignment, retained-noise
recoverability, and reconstruction compatibility. Under realizability,
Proposition~\ref{prop:app_image_only_optimum} and
Lemma~\ref{lem:app_image_only_invariance} establish an audit-invariant
image-only solution when that solution is representable within the complete
segmentation-output predictor class. Path Category then determines whether
the documented objective provides a segmentation-supervised route to this
solution without first requiring an accurate noised-quantity prediction.
This gives the structural basis for the Preserved hypothesis in LC-Hybrid
and \pathSegTargeted{} routes.

For Gaussian \(\epsilon\)-prediction with squared-error noise supervision,
Propositions~\ref{prop:app_full_information} and
\ref{prop:app_gaussian_risks} characterize retained-noise recoverability
under the audit constructions, while
Propositions~\ref{prop:app_image_only_reconstruction} and
\ref{prop:app_reconstruction_compatibility} characterize the
supplied-state-dependent intermediate and its compatibility with
clean-target reconstruction. Together with the documented supervision
paths, these results provide the formal basis for the SR hypothesis in
Diff-Coupled routes and in HC-Hybrid routes whose segmentation supervision
is applied to the analyzed reconstruction. Preserved and SR themselves are
empirical decisions determined by finite retraining, held-out Dice, and the
reference-variability margin.

The direct formal scope concerns the training-time supplied-state
constructions analyzed above. For multistep methods, the model-generated
reverse-chain state at inference is not identified samplewise with the Full
training state at the same nominal timestep, and additional objectives,
fusion mechanisms, and auxiliary conditioning remain method-specific. For a
deterministic clean latent \(\Zzero=T(\Yzero)\), the segmentation-risk and
Gaussian reconstruction arguments extend when the corresponding latent
assumptions hold. Bernoulli, score, velocity, posterior, and other prediction
targets or variational objectives require target- and loss-specific analyses;
their inclusion in the experiments therefore tests whether the route-level
response pattern extends beyond the Gaussian setting analyzed formally.

\section{Noise-Loss Removal and Margin Sensitivity}
\begin{table*}[!t]
\label{app:margin_sensitivity}
\centering
\small
\renewcommand{\arraystretch}{0.97}
\caption{Noise-loss-removal ablation performance for the four audited methods natively classified as HC-Hybrid. The noise loss was removed while all other components and settings were held fixed. Dice (\%) is reported as mean \(\pm\) SD over ten matched seeds.}
\label{tab:hc_no_noise_perf}
\begin{tabular*}{\textwidth}{@{\extracolsep{\fill}}cl|c|c|c|c@{}}
\toprule
Dataset & Method &
\makecell{Original Full Dice} &
\makecell{No-Noise-Loss Full Dice} &
\makecell{Random-$Y_t$ Dice} &
\makecell{Shuffle-$Y_t$ Dice} \\
\midrule
\multirow[c]{4}{*}{\textbf{BTCV}} & \methodMedSegDiffVTwo & 81.19$\pm$1.75 & 80.52$\pm$1.15 & 11.05$\pm$1.72 & 10.22$\pm$1.34 \\
& \methodLDSeg & 64.26$\pm$2.06 & 63.03$\pm$1.98 & 19.06$\pm$1.71 & 51.01$\pm$1.55 \\
& \methodSDSeg & 90.56$\pm$1.57 & 89.02$\pm$1.18 & 76.98$\pm$1.42 & 81.46$\pm$1.39 \\
& \methodTSLDSeg & 82.01$\pm$1.60 & 82.87$\pm$1.85 & 13.84$\pm$2.03 & 18.84$\pm$1.51 \\
\cmidrule(lr){1-6}
\multirow[c]{4}{*}{\textbf{ACDC}} & \methodMedSegDiffVTwo & 72.25$\pm$1.54 & 69.85$\pm$1.37 & 6.96$\pm$1.17 & 7.29$\pm$1.24 \\
& \methodLDSeg & 62.65$\pm$1.60 & 62.52$\pm$1.50 & 24.87$\pm$1.20 & 35.18$\pm$1.42 \\
& \methodSDSeg & 80.85$\pm$1.75 & 79.31$\pm$1.13 & 4.14$\pm$1.17 & 4.27$\pm$1.16 \\
& \methodTSLDSeg & 77.03$\pm$1.69 & 77.52$\pm$1.33 & 3.82$\pm$1.47 & 3.82$\pm$1.44 \\
\cmidrule(lr){1-6}
\multirow[c]{4}{*}{\textbf{ISIC2018}} & \methodMedSegDiffVTwo & 83.52$\pm$1.72 & 81.65$\pm$1.51 & 43.60$\pm$1.49 & 42.77$\pm$1.39 \\
& \methodLDSeg & 85.70$\pm$1.77 & 83.38$\pm$1.33 & 71.65$\pm$1.44 & 77.41$\pm$1.11 \\
& \methodSDSeg & 89.13$\pm$1.53 & 88.48$\pm$1.11 & 29.60$\pm$1.29 & 52.55$\pm$1.06 \\
& \methodTSLDSeg & 90.30$\pm$1.74 & 89.32$\pm$1.07 & 29.39$\pm$1.41 & 49.75$\pm$1.70 \\
\bottomrule
\end{tabular*}
\end{table*}

\begin{table*}[!t]
\centering
\small
\setlength{\tabcolsep}{3pt}
\caption{Noise-loss-removal ablation decisions at \(\kappa=1.0\) from ten seed-matched pairs. \(\Delta\)Dice is audit condition minus No-Noise-Loss Full, and \(\delta_r\) is the seed-level SD of No-Noise-Loss Full. Mean paired \(\Delta\)Dice is reported as \(\bar{\Delta}\pm t_{0.95,9}\mathrm{SE}_{\Delta}\), corresponding to the 90\% CI \([\mathrm{LCB}_{95},\mathrm{UCB}_{95}]\). SR denotes supported state reliance.}
\label{tab:hc_no_noise_stat}
\begin{tabular*}{\textwidth}{@{\extracolsep{\fill}}l|c|cc|cc@{}}
\toprule
\multirow{2}{*}{Method} & \multirow{2}{*}{\(\delta_r\) (pp)} &
\multicolumn{2}{c|}{Random-\(Y_t\)} &
\multicolumn{2}{c}{Shuffle-\(Y_t\)} \\
& & Mean \(\Delta\)Dice (pp) & Decision & Mean \(\Delta\)Dice (pp) & Decision \\
\midrule
\multicolumn{6}{@{}l}{\textbf{BTCV}} \\[-0.3ex]
\methodMedSegDiffVTwo & 1.15 & $-69.47\pm1.15$ & SR & $-70.30\pm0.91$ & SR \\
\methodLDSeg & 1.98 & $-43.97\pm1.24$ & SR & $-12.02\pm1.37$ & SR \\
\methodSDSeg & 1.18 & $-12.04\pm0.31$ & SR & $-7.56\pm0.33$ & SR \\
\methodTSLDSeg & 1.85 & $-69.03\pm0.16$ & SR & $-64.03\pm0.43$ & SR \\
\cmidrule(lr){1-6}
\multicolumn{6}{@{}l}{\textbf{ACDC}} \\[-0.3ex]
\methodMedSegDiffVTwo & 1.37 & $-62.89\pm1.20$ & SR & $-62.56\pm1.24$ & SR \\
\methodLDSeg & 1.50 & $-37.65\pm1.05$ & SR & $-27.34\pm0.98$ & SR \\
\methodSDSeg & 1.13 & $-75.17\pm1.03$ & SR & $-75.04\pm0.98$ & SR \\
\methodTSLDSeg & 1.33 & $-73.70\pm0.33$ & SR & $-73.69\pm0.32$ & SR \\
\cmidrule(lr){1-6}
\multicolumn{6}{@{}l}{\textbf{ISIC2018}} \\[-0.3ex]
\methodMedSegDiffVTwo & 1.51 & $-38.05\pm1.33$ & SR & $-38.88\pm1.26$ & SR \\
\methodLDSeg & 1.33 & $-11.73\pm0.46$ & SR & $-5.97\pm0.27$ & SR \\
\methodSDSeg & 1.11 & $-58.88\pm0.40$ & SR & $-35.94\pm0.59$ & SR \\
\methodTSLDSeg & 1.07 & $-59.94\pm1.01$ & SR & $-39.57\pm1.31$ & SR \\
\bottomrule
\end{tabular*}
\end{table*}

\begin{table*}[!t]
\centering
\small
\caption{Comparisons whose decisions change under \(\delta_r(\kappa)=\kappa\,\mathrm{SD}_{r}\). Boldface marks the primary analysis at \(\kappa=1.0\).}
\label{tab:margin_sensitivity}
\begin{tabular*}{\textwidth}{@{\extracolsep{\fill}}ll|l|cccc@{}}
\toprule
Dataset & Method & Audit & 0.5 & 1.0 & 1.5 & 2.0 \\
\midrule
BTCV & \methodTSLDSeg & \corenodiff{} & Better & \textbf{Better} & Similar & Similar \\
BTCV & \methodEnsemDiff & \corenodiff{} & Worse & \textbf{Worse} & Similar & Similar \\
BTCV & \methodColdSegDiff & \corenodiff{} & Inconclusive & \textbf{Similar} & Similar & Similar \\
ACDC & \methodLEAF & \corenodiff{} & Better & \textbf{Inconclusive} & Similar & Similar \\
ISIC2018 & \methodMedSegDiffVTwo & \corenodiff{} & Better & \textbf{Better} & Similar & Similar \\
ISIC2018 & \methodSDSeg & \corenodiff{} & Better & \textbf{Better} & Similar & Similar \\
ISIC2018 & \methodLDSeg & \corenodiff{} & Better & \textbf{Better} & Better & Similar \\
ISIC2018 & \methodEnsemDiff & \corenodiff{} & Worse & \textbf{Inconclusive} & Similar & Similar \\
ISIC2018 & \methodColdSegDiff & \corenodiff{} & Better & \textbf{Inconclusive} & Similar & Similar \\
\bottomrule
\end{tabular*}
\end{table*}

\end{document}